\documentclass[11pt]{article}

\PassOptionsToPackage{numbers,sort&compress}{natbib}
\usepackage{natbib}
\usepackage[utf8]{inputenc}
\usepackage[T1]{fontenc}
\usepackage{microtype}
\usepackage{url}
\usepackage{xurl}
\usepackage{booktabs}
\usepackage{amsmath,amssymb,amsfonts}
\usepackage{amsthm}
\usepackage{graphicx}
\usepackage{multirow}
\usepackage{array}
\usepackage{enumitem}
\usepackage{algorithm}
\usepackage{algpseudocode}
\usepackage{longtable}
\usepackage{caption}
\usepackage{subcaption}
\usepackage{xcolor}
\usepackage{float}
\usepackage[hidelinks]{hyperref}

\usepackage[margin=1in]{geometry}

\setlist[itemize]{leftmargin=*,topsep=2pt,itemsep=1pt}
\setlist[enumerate]{leftmargin=*,topsep=2pt,itemsep=1pt}

\newcommand{\R}{\mathbb{R}}
\newcommand{\Sphere}{\mathbb{S}^{2}}
\newcommand{\abs}[1]{\left|#1\right|}
\newcommand{\norm}[1]{\left\lVert #1 \right\rVert}
\newcommand{\med}{\operatorname{median}}
\newcommand{\TriDE}{TriDE}
\newcommand{\eMean}{\ensuremath{\bar e}}
\newcommand{\eMed}{\ensuremath{\tilde e}}
\newcommand{\eP}{\ensuremath{e_{90}}}
\newcommand{\tMean}{\ensuremath{\bar t}}
\newcommand{\tMed}{\ensuremath{\tilde t}}
\newcommand{\tP}{\ensuremath{t_{90}}}
\newcommand{\tauDir}{\ensuremath{\tau_{\mathrm{dir}}}}
\newcommand{\dTauDir}{\ensuremath{\Delta\tau_{\mathrm{dir}}}}

\providecommand{\Sph}{\mathbb{S}}
\providecommand{\Prob}{\mathbb{P}}

\providecommand{\err}{\operatorname{err}}
\providecommand{\argmin}{\operatorname*{arg\,min}}

\newtheorem{theorem}{Theorem}[section]
\newtheorem{lemma}[theorem]{Lemma}
\newtheorem{corollary}[theorem]{Corollary}
\newtheorem{definition}[theorem]{Definition}

\newtheorem{assumption}[theorem]{Assumption}

\title{\vspace{-1.5em}
TriDE: Triangle-Consistent Translation Directions for Global Camera Pose Estimation
\vspace{-0.5em}}

\author{
Francisco Chen \quad Yiran Wang \quad Yunpeng Shi\thanks{Corresponding author. ypshi@ucdavis.edu}\\
Department of Mathematics\\
University of California, Davis
}

\date{}

\begin{document}

\maketitle

\begin{abstract}
Pairwise translation directions are a key input to camera location estimation in global structure-from-motion. Existing estimators usually process each image pair independently, producing directions that may be locally plausible but inconsistent with the other relative directions in the viewing graph. To jointly estimate the direction, we propose \emph{TriDE}, which exploits camera-triangle consistency as an efficient higher-order verification signal. Instead of solving a costly global nonlinear optimization problem that is sensitive to initialization, TriDE refines unreliable pairwise directions through message passing between directions and their incident weighted triangles. This information propagation strategy enables us to establish a strong phase-transition bound for exact recovery under a realistic random corruption model. Experiments on real image graphs show that TriDE improves direction accuracy by a large margin and yields better downstream camera locations, providing a practical link between local pairwise estimation and global camera pose geometry.
\end{abstract}

\section{Introduction}
Structure-from-motion (SfM) reconstructs camera motion and sparse 3D scene structure from overlapping images~\citep{hartley2003multiple,schonberger2016structure}. In global SfM, an image collection is represented as a viewing graph, whose nodes are cameras and whose edges encode pairwise image matches. A typical pipeline first synchronizes rotations, then estimates pairwise translation directions, solves for camera locations by translation averaging, and finally triangulates and refines 3D points~\citep{martinec2007robust,cui2015global,pan2024global}. Translation-direction estimation therefore sits at a critical interface between local image matching and global camera localization: errors at this stage can propagate through translation averaging and degrade the camera geometry used by all subsequent SfM stages.

The central difficulty is that translation directions are inferred from local image evidence but must be consistent in the view graph. In structured scenes, repeated patterns and approximate symmetries, such as similar windows or facade elements, can make incorrect matches locally plausible. A relative direction may explain the keypoint correspondences for a single image pair while being incompatible with the directions on adjacent edges. Thus, reliable global localization requires using both pointwise correspondence evidence on each edge and higher-order geometric consistency across the graph.

Mathematically, for an image pair $(i,j)$ with camera centers $\mathbf{y}_i,\mathbf{y}_j\in\R^3$, the relative translation direction is
\begin{equation}
    \mathbf{g}_{ij}
    =
    \frac{\mathbf{y}_i-\mathbf{y}_j}{\norm{\mathbf{y}_i-\mathbf{y}_j}}
    \in\Sphere,
    \label{eq:intro-direction}
\end{equation}
where $\mathbf{g}_{ij}$ and $-\mathbf{g}_{ij}$ represent the same unoriented line. Classical pairwise methods estimate a fundamental or essential matrix from keypoint correspondences and then extract the relative translation direction~\citep{hartley2003multiple,hartley1997defense,nister2004efficient}. These methods are simple and widely used, but they treat each image pair independently and do not revisit the estimated directions after absolute rotations have been synchronized.

Rotation synchronization provides additional structure that can be exploited. Once all measurements are expressed in a common coordinate frame, each matched keypoint pair induces a correspondence normal $\mathbf{x}_{ij,r}\in\Sphere$. In the noiseless case, an inlier correspondence satisfies
\begin{equation}
    \mathbf{g}_{ij}^{\top}\mathbf{x}_{ij,r}=0 .
    \label{eq:intro-point-orthogonality}
\end{equation}
Hence $\mathbf{g}_{ij}$ can be estimated as the one-dimensional normal to the two-dimensional subspace spanned by inlier correspondence normals. Robust subspace recovery (RSR) methods exploit this pointwise structure and can produce strong edge-level direction estimates~\citep{ozyesil2015robust,lerman2018fast,yu2024subspace}. However, they remain edge-local: they can select a direction well supported by the matches on one image pair without checking whether it agrees with neighboring cameras.

The missing cue is higher-order geometry in the viewing graph. For any nondegenerate camera triangle $(i,j,k)$, the noiseless directions $\mathbf{g}_{ij}$, $\mathbf{g}_{jk}$, and $\mathbf{g}_{ik}$ are coplanar up to orientation; equivalently, their absolute scalar triple product vanishes. Triangle coplanarity directly couples adjacent translation directions and provides a graph-level consistency test that is unavailable to purely pairwise estimators. Existing view-graph filtering and translation-averaging methods also exploit global consistency, but mainly by pruning or reweighting unreliable edges, correspondences, or constraints~\citep{wilson2014robust,shi2018estimation,shah2018view,manam2022correspondence}. Such strategies can suppress severe outliers, but they may also reduce connectivity and discard edges whose raw correspondences still contain useful signal.

We propose \TriDE{}, a translation-direction estimation method that turns higher-order view-graph consistency into a mechanism for refining pairwise directions. Instead of treating unreliable measurements only as edges or correspondences to be downweighted or removed, \TriDE{} revisits the correspondence evidence on each edge and uses neighboring camera triangles to identify directions that are more globally compatible. This allows poor initial directions to be corrected when their raw matches still contain useful signal. The resulting direction estimates are more consistent with the surrounding graph, while the original view graph is kept intact for downstream translation averaging.

\section{Related Work}

Pairwise translation directions are traditionally obtained by estimating a fundamental matrix, or an essential matrix for calibrated cameras, from feature correspondences and then decomposing the relative pose~\citep{nister2004efficient,fischler1981random,barath2019magsac,barath2020magsac++}. These methods are modular and effective when the inlier set is clean, but the translation direction is only a byproduct of pairwise epipolar fitting. As a result, the estimate can be sensitive to outlier ratios, noise-scale choices, low-parallax or nearly planar configurations, and local model selection. More importantly, these methods use only image-pair information and do not exploit the graph context available in global SfM.

After rotations are synchronized, robust subspace recovery gives a more direct route to translation-direction estimation. Each correspondence normal should be orthogonal to the unknown direction, so least-squares PCA or robust alternatives such as least-unsquared deviations, Fast Median Subspace (FMS), and Subspace-constrained Tyler's estimator (STE) can recover the direction from an inlier-dominated subspace~\citep{ozyesil2015robust,lerman2018fast,lerman2018exact,yu2024subspace}. These methods use the keypoint measurements more directly than epipolar decomposition and provide strong local initializers. Their limitation is not local robustness, but the absence of graph awareness: the selected direction is judged by its support on one edge, not by its compatibility with adjacent edges.

Global SfM and translation-averaging methods use view-graph structure throughout the reconstruction pipeline. Rotation averaging enforces consistency among relative rotations~\cite{shi2020message,lerman2022robust,shi2022robust,chatterjee2013efficient}, while translation averaging estimates camera locations from relative directions or displacement constraints~\cite{ozyesil2015robust,wilson2014robust,lerman2018exact,shi2018estimation,zhuang2018baseline,manam2024fusing}. Robust pipelines further improve reliability by reweighting correspondences, filtering view graphs, or downweighting unreliable constraints before or during global optimization~\citep{martinec2007robust,cui2015global,wilson2014robust,shi2018estimation,shah2018view,manam2022correspondence,pan2024global}. These works demonstrate the importance of graph-level consistency for robust SfM, but they primarily use it to stabilize downstream estimation from fixed pairwise measurements rather than as feedback for refining the translation directions estimated on each edge.

A more direct approach is to solve a global nonlinear problem over all pairwise directions, for example by combining point residuals with triangle determinant constraints $\det(\mathbf{g}_{ij},\mathbf{g}_{jk},\mathbf{g}_{ik})=0$ on graph triangles. Tangent-space Gauss--Newton (GN) and Levenberg--Marquardt (LM) methods can be applied to such objectives~\citep{levenberg1944method,marquardt1963algorithm}. However, the optimization is nonconvex on a product of spheres, and the determinant residuals are multilinear in the unknown directions. Continuous solvers such as \textsc{GN} and \textsc{LM} therefore tend to follow a single local trajectory from the supplied initialization and can be sensitive to poor starting directions. We use them mainly as diagnostic baselines.

\TriDE{} complements these approaches by using graph consistency for candidate verification during direction estimation. It generates plausible directions from correspondence normals and selects among them using triangle coplanarity, rather than relying solely on local pairwise fitting, graph pruning, or a fully coupled nonlinear formulation. This enables an inaccurate edge direction to be replaced when a better correspondence-induced candidate agrees with neighboring camera triples, while preserving the view graph for downstream translation averaging.

Our contributions are summarized as follows:
\begin{enumerate}[leftmargin=*,nosep]
    \item We introduce TriDE, a graph-preserving translation-direction refinement method that combines correspondence-level evidence with triangle-level consistency.

    \item We show how triangle consistency can be used as a bounded candidate-verification signal, allowing unreliable directions to be repaired without pruning the view graph or solving a large coupled nonconvex problem.

    \item We provide a one-sweep exact-recovery analysis under a structured recoverability model, yielding interpretable clean-triangle coverage thresholds for complete, Erdős-Rényi, and random geometric graphs.

    \item We demonstrate on ETH3D image graphs that TriDE improves direction accuracy across multiple local initializers and improves downstream camera location estimation with small additional runtime.
\end{enumerate}

\section{Problem Setup}

Let $G=(V,E)$ be the view graph, where $V=\{1,\ldots,n_{\mathrm{cam}}\}$ is the set of camera indices and $E\subseteq\{\{i,j\}:i,j\in V,\ i\ne j\}$ is the set of matched image pairs. For each edge $e=(i,j)\in E$, we seek the unoriented pairwise translation direction $\mathbf{g}_{ij}\in\Sphere$ defined in Eq.~\eqref{eq:intro-direction}; we also write $\mathbf{g}_e\equiv\mathbf{g}_{ij}$. All directions and correspondence normals are expressed in a common frame after a rotation backend synchronizes the camera rotations~\citep{chatterjee2013efficient,shi2020message}. With camera-frame-to-world rotation $\mathbf{R}_i\in\mathrm{SO}(3)$, calibration matrix $\mathbf{K}_i\in\R^{3\times 3}$, and homogeneous measurement $\tilde{\mathbf{u}}_{i,r}\in\R^3$, a matched keypoint pair gives bearings $\mathbf{b}_{i,r},\mathbf{b}_{j,r}\in\Sphere$:
\begin{equation}
    \mathbf{b}_{i,r}=\frac{\mathbf{R}_i\mathbf{K}_i^{-1}\tilde{\mathbf{u}}_{i,r}}{\norm{\mathbf{R}_i\mathbf{K}_i^{-1}\tilde{\mathbf{u}}_{i,r}}},
    \qquad
    \mathbf{b}_{j,r}=\frac{\mathbf{R}_j\mathbf{K}_j^{-1}\tilde{\mathbf{u}}_{j,r}}{\norm{\mathbf{R}_j\mathbf{K}_j^{-1}\tilde{\mathbf{u}}_{j,r}}} .
    \label{eq:bearings}
\end{equation}
Following pairwise direction-estimation formulations from epipolar geometry~\citep{hartley2003multiple,ozyesil2015robust}, the stored correspondence normal is
\begin{equation}
    \mathbf{x}_{ij,r}=\frac{\mathbf{b}_{i,r}\times \mathbf{b}_{j,r}}{\norm{\mathbf{b}_{i,r}\times \mathbf{b}_{j,r}}}\in\Sphere,
    \label{eq:correspondence-normal}
\end{equation}
with zero cross-products discarded. Because edge orientation or match order may flip signs, all residuals are sign-invariant. An inlier correspondence satisfies Eq.~\eqref{eq:intro-point-orthogonality} up to noise. Let $\mathcal{X}_e=\{\mathbf{x}_{e,r}\}_{r=1}^{n_e}\subset\Sphere$ be the local evidence set for edge $e$, where $n_e=|\mathcal{X}_e|$ is the number of stored correspondence normals, and let $\mathbf{g}_e^{(0)}\in\Sphere$ be an initial direction from a compatible local estimator.

Let $\mathcal{T}$ denote the set of graph triangles. For a triangle $\tau=(i,j,k)\in\mathcal{T}$ with edges $e_{ij}$, $e_{jk}$, and $e_{ik}$, noiseless directions in the common frame are coplanar up to sign:
\begin{equation}
    \det(\mathbf{g}_{ij},\mathbf{g}_{jk},\mathbf{g}_{ik}) = 0 .
    \label{eq:triangle-det}
\end{equation}
\TriDE{} uses absolute determinant residuals as a sign invariant verification cue. This cue alone is not sufficient, since incorrect candidates can also be coplanar, especially under nearly planar motion. Therefore, local measurements and a coarse initial direction graph remain important. At the same time, because the candidate pool is generated from raw correspondence normals rather than only by tangent descent from the current direction, \TriDE{} can move an edge away from a poor initialization when a better sampled candidate has both point support and triangle support. For downstream location averaging, all methods receive the same cheirality style sign vote after direction estimation. No ground truth signs enter the \TriDE{} loop.

\section{TriDE: Message Passing Among Keypoints, Edges, and Triangles}
\label{sec:tride}
Instead of solving a costly joint optimization problem, \TriDE{} propagates information efficiently across keypoints, pairwise directions, and camera triangles.  It takes as input a fixed viewing graph, synchronized rotations, per-edge correspondence normals, and initial pairwise direction estimates.  The graph and rotations remain fixed throughout.  During each sweep, \TriDE{} updates direction candidates using triangle consistency, weights triangles by the current reliability of their neighboring edges, and then refreshes edge reliability from the original keypoint measurements.  In this way, correspondence-level support, edge-level reliability, and triangle-level consistency exchange information and reinforce one another.

\subsection{Keypoints $\rightarrow$ Edge: Message Passing for Edge Reliability}
At sweep \(t\), edge \(e\) retains its current direction and augments it with up
to \(n_{\mathrm{cand}}\) random two-normal hypotheses. This follows the standard
observation that two inlier correspondence normals determine a candidate
translation direction through their orthogonal complement
\citep{ozyesil2015robust,lerman2018fast}. For a sampled index-pair set
\(\mathcal{P}_e^{(t)}\subseteq\{1,\ldots,n_e\}^2\), the candidate set
\(\mathcal{C}_e^{(t)}\subset\Sphere\) is
\begin{equation}
    \mathcal{C}_e^{(t)} = \{\mathbf{g}_e^{(t)}\}\cup
    \left\{\frac{\mathbf{x}_{e,r_1}\times \mathbf{x}_{e,r_2}}
    {\norm{\mathbf{x}_{e,r_1}\times \mathbf{x}_{e,r_2}}}:
    (r_1,r_2)\in\mathcal{P}_e^{(t)},\;
    \norm{\mathbf{x}_{e,r_1}\times \mathbf{x}_{e,r_2}}>0\right\}.
    \label{eq:candidate-pool}
\end{equation}
Each cross product gives a locally plausible edge direction. Keeping
\(\mathbf{g}_e^{(t)}\) provides a fallback candidate, while the bounded pool
makes the limitation explicit: if neither the initializer nor the sampled
two-normal hypotheses contains a good direction, \TriDE{} cannot recover it.

For an edge direction $\mathbf{g}\in\Sphere$, define the unoriented angular point residual
\begin{equation}
    r_{e,r}(\mathbf{g})=\arcsin\abs{\mathbf{g}^\top \mathbf{x}_{e,r}}\in\R_{\ge 0} .
    \label{eq:angular-point-residual}
\end{equation}
The point support and badness score are the scalars
\begin{equation}
    a^{\mathrm{pt}}_e(\mathbf{g})
    =\frac{1}{n_e}\sum_{r=1}^{n_e}
    \exp\!\left(-\frac{r_{e,r}(\mathbf{g})^2}{2\sigma^2}\right),
    \qquad
    s_e(\mathbf{g})=1-a^{\mathrm{pt}}_e(\mathbf{g}) .
    \label{eq:point-support}
\end{equation}
The point support \(a^{\mathrm{pt}}_e(\mathbf{g})\) quantifies how well the direction aligns with the normal vectors associated with the correspondences. Lower $s_e(\mathbf{g})$ means stronger local support. We initialize and refresh this score as
\begin{equation}
    s_e^{(0)} = s_e\!\left(\mathbf{g}_e^{(0)}\right),
    \label{eq:init-badness}
\end{equation}
\begin{equation}
    s_e^{(t+1)} = s_e\!\left(\mathbf{g}_e^{(t+1)}\right)
    = 1-a^{\mathrm{pt}}_e\!\left(\mathbf{g}_e^{(t+1)}\right).
    \label{eq:refreshed-badness}
\end{equation}

\subsection{Edge $\rightarrow$ Triangle: Message Passing for Triangle Reliability}

We first introduce the triangle-consistency constraint.  For an edge \(e\) in a triangle \(\tau\), let \(a\) and \(b\) denote the other two edges.  Using their current directions, define the triangle normal
\begin{equation}
    \mathbf{n}_{\tau,e}^{(t)}
    =
    \frac{\mathbf{g}_a^{(t)}\times \mathbf{g}_b^{(t)}}
    {\norm{\mathbf{g}_a^{(t)}\times \mathbf{g}_b^{(t)}}}.
    \label{eq:triangle-normal}
\end{equation}
When the three directions on \(\tau\) are clean and nondegenerate, the true direction on \(e\) is orthogonal to this normal.  Thus clean incident triangles provide normals lying in the two-dimensional subspace orthogonal to \(\mathbf{g}_e^\star\).  However, not every incident triangle is reliable: if either neighboring edge \(a\) or \(b\) is corrupted, then the resulting normal may give misleading information for updating edge \(e\).  We therefore weight each incident triangle by the current reliability or noise level of its two neighboring edges:
\begin{equation}
    \alpha_{\tau,e}^{(t)}
    =
    \exp[-\beta(s_a^{(t)}+s_b^{(t)})],
    \label{eq:raw-weight}
\end{equation}
and normalize over all valid incident triangles,
\begin{equation}
    w_{\tau,e}^{(t)}
    =
    \frac{\alpha_{\tau,e}^{(t)}}
    {\sum_{\tau'\in\mathcal{T}_e^{(t)}}\alpha_{\tau',e}^{(t)}},
    \label{eq:weights}
\end{equation}
where \(\mathcal{T}_e^{(t)}\subseteq\mathcal{T}\) is the set of valid incident triangles for edge \(e\). The weight \(w_{\tau,e}^{(t)}\) measures how much the triangle \(\tau\) should be trusted when updating edge \(e\).  Triangles whose two neighboring edges have low badness receive larger weights, while triangles involving unreliable neighboring directions are downweighted.

\subsection{Triangle $\rightarrow$ Keypoints: Message Passing for Keypoint Selection}
Each candidate direction \(\mathbf{c}\in\mathcal{C}_e^{(t)}\), generated either from the current direction or from a candidate keypoint-match pair, is scored by the weighted triangle inconsistency
\begin{equation}
    \ell_e^{(t)}(\mathbf{c})
    =
    \sum_{\tau\in\mathcal{T}_e^{(t)}}
    w_{\tau,e}^{(t)}
    \abs{\mathbf{c}^\top \mathbf{n}_{\tau,e}^{(t)}} .
    \label{eq:weighted-score}
\end{equation}
A small value of \(\ell_e^{(t)}(\mathbf{c})\) means that \(\mathbf{c}\) is nearly orthogonal to the reliable triangle normals incident to \(e\), and is therefore more compatible with the surrounding graph geometry.

For each edge with at least one valid triangle context, \TriDE{} selects
\begin{equation}
    \mathbf{g}_e^{(t+1)}
    \in
    \arg\min_{\mathbf{c}\in\mathcal{C}_e^{(t)}}
    \ell_e^{(t)}(\mathbf{c}),
    \label{eq:direction-update}
\end{equation}
and then refreshes \(s_e^{(t+1)}\) according to Eq.~\eqref{eq:refreshed-badness}.  This closes one loop of information propagation, and the resulting message-passing procedure is repeated for a few sweeps.  The updates are synchronous: all edge updates in sweep \(t\) are computed from \(\{\mathbf{g}_e^{(t)}\}_{e\in E}\) and \(\{s_e^{(t)}\}_{e\in E}\), and only then are the new states \(\{\mathbf{g}_e^{(t+1)}\}_{e\in E}\) and \(\{s_e^{(t+1)}\}_{e\in E}\) formed. \TriDE{} is summarized in Algorithm~\ref{alg:tride}.
\begin{algorithm}[H]
\caption{TriDE}
\label{alg:tride}
\begin{algorithmic}[1]
\Require View graph $G=(V,E)$, normal sets $\{\mathcal{X}_e\}_{e\in E}$, initial directions $\{\mathbf{g}_e^{(0)}\}_{e\in E}$, candidate budget $n_{\mathrm{cand}}$, sharpness $\beta$, support scale $\sigma$, maximum sweeps $k_{\max}$, tolerance $\tau_{\mathrm{stop}}$
\State Initialize $s_e^{(0)}$ by Eq.~\eqref{eq:init-badness}; precompute graph triangles.
\For{$t=0,1,\ldots,k_{\max}-1$}
    \For{each edge $e\in E$}
        \State Build $\mathcal{C}_e^{(t)}$ by Eq.~\eqref{eq:candidate-pool}.
        \State Build valid triangle normals by Eq.~\eqref{eq:triangle-normal}.
        \If{edge $e$ has no valid triangle context}
            \State Set $\mathbf{g}_e^{(t+1)}\leftarrow \mathbf{g}_e^{(t)}$ and $s_e^{(t+1)}\leftarrow s_e^{(t)}$.
        \Else
            \State Weight triangles by Eqs.~\eqref{eq:raw-weight}--\eqref{eq:weights}.
            \State Select $\mathbf{g}_e^{(t+1)}$ by Eq.~\eqref{eq:direction-update}.
            \State Refresh $s_e^{(t+1)}\leftarrow s_e(\mathbf{g}_e^{(t+1)})$.
        \EndIf
    \EndFor
    \State $\delta_t\leftarrow\med_{e\in E}\arccos\abs{(\mathbf{g}_e^{(t+1)})^\top \mathbf{g}_e^{(t)}}$.
    \If{$t\ge 1$ and $\delta_t < \tau_{\mathrm{stop}}$}
        \State \textbf{break}
    \EndIf
\EndFor
\State \Return estimated directions $\{\mathbf{g}_e\}_{e\in E}$ and refreshed badness diagnostics $\{s_e\}_{e\in E}$.
\end{algorithmic}
\end{algorithm}

\subsection{Implementation and complexity}
\label{sec:complexity}

Let $m=|E|$, $n_{\triangle}=|\mathcal{T}|$, $d_e^{(t)}=|\mathcal{T}_e^{(t)}|$, and $n_{\mathrm{corr}}=\sum_{e\in E}n_e$. Graph triangles are enumerated once from the fixed edge list; the reported runs use the full triangle set.

Each sweep constructs at most $n_{\mathrm{cand}}$ candidates per edge, touches each valid edge--triangle incidence for triangle construction and weighting, scores candidates against incident triangles, and refreshes selected directions against their local measurements. Since $\sum_e d_e^{(t)}\le 3n_{\triangle}$,
\begin{equation}
    O\!\left(\sum_{e\in E} |\mathcal{C}_e^{(t)}|d_e^{(t)}\right)
    \le O\!\left((n_{\mathrm{cand}}+1)n_{\triangle}\right) .
\end{equation}
For $k_{\max}$ sweeps, the total estimation-stage cost is
\begin{equation}
    O\!\left(k_{\max}\big((n_{\mathrm{cand}}+1)n_{\triangle}+n_{\mathrm{corr}}+n_{\mathrm{cand}}m\big)\right),
\end{equation}
with memory $O(n_{\mathrm{corr}}+m+n_{\triangle})$.

\section{Theory at a Glance}
\label{sec:theory-at-a-glance}

This section summarizes the exact-recovery message for \TriDE{}; formal assumptions and proofs are deferred to Appendix~\ref{app:theory}.  The guarantee is a direction-refinement result: with fixed rotations and candidate pools, one sufficiently sharp \TriDE{} sweep recovers all pairwise translation directions, while downstream location recovery still requires the usual sign-orientation and bearing-rigidity conditions.

The key mechanism is the measurement-refreshed triangle weight.  At support scale \(\sigma_0=1^\circ\), write \(A_e^{\sigma_0}(g)=a_e^{\rm pt}(g)\) and \(s_e(g)=1-A_e^{\sigma_0}(g)\).  For an edge \(ij\) and triangle witness \(k\),
\[
        \exp[-\beta(s_{ik}+s_{jk})]
        \propto
        \exp\{\beta(A_{ik}^{\sigma_0}(g_{ik})+A_{jk}^{\sigma_0}(g_{jk}))\}.
\]
Thus large \(\beta\) makes \TriDE{} favor witnesses whose two neighboring directions have strong point support.

The theory uses a two-class initialization model.  Clean-anchor edges are initialized at the true direction and have high point support; weak edges may be initialized incorrectly but are recoverable if their candidate pools contain the true direction.  Suppose every edge has an \(a\)-fraction of clean-clean witnesses, these witnesses are \(c_{\rm wd}\)-well-distributed, clean witnesses have an effective support gap \(\Delta_{\sigma_0}^{\rm eff}>0\), and every false candidate is separated from the true direction by at least \(\eta\).  Then one \TriDE{} sweep satisfies
\[
        \max_{ij\in E}
        \operatorname{err}(g_{ij}^{+},g_{ij}^{\star})
        \le
        \frac{1-a}{a c_{\rm wd}}
        \exp(-\beta\Delta_{\sigma_0}^{\rm eff}).
\]
Thus the one-sweep error is exponentially small for sufficiently large \(\beta\).  Once this error is below the candidate separation \(\eta\), every edge selects the true direction, up to sign. In practice, we do not take \(\beta\to\infty\).  A moderately large value already makes the softmax weights sharply concentrate on high-support triangle witnesses, while avoiding numerical instability in the unnormalized weights \(\exp(\beta H)\).  We use \(\beta=15\) in our experiments.

\begin{table}[H]
\centering
\caption{Informal one-sweep exact-recovery thresholds for \TriDE{}.  Here \(p\) is the Erd\H{o}s--R\'enyi edge probability, \(q\) is the weak-edge probability, and \(r\) is the connection radius of a 3D random geometric graph.  Constants hide well-distributedness, support-gap, candidate-separation, and failure-probability factors.}
\label{tab:tride-phase-transition}
\scriptsize
\begin{tabular}{@{}lccc@{}}
\toprule
Model & Clean-clean witnesses & Sufficient condition & Allowed \(q\) \\
\midrule
Complete graph \(K_n\)
& \(n(1-q)^2\)
& \(n(1-q)^2\gtrsim \log n\)
& \(q\le 1-C\sqrt{\log n/n}\) \\
Erd\H{o}s--R\'enyi \(G(n,p)\)
& \(np^2(1-q)^2\)
& \(np^2(1-q)^2\gtrsim \log |E|\)
& \(q\le 1-C\sqrt{\log |E|/(np^2)}\) \\
3D random geometric graph
& \(nr^3(1-q)^2\)
& \(nr^3(1-q)^2\gtrsim \log n\)
& \(q\le 1-C\sqrt{\log n/(nr^3)}\) \\
\bottomrule
\end{tabular}
\end{table}

The table isolates the graph-coverage requirement.  The full recovery result also requires local candidate recall and a support gap: each recoverable edge must contain the true direction in its candidate pool, and weak directions must not have the same high point support as clean anchors.  With \(B\) random two-normal samples per edge, a sufficient recall condition is
\[
        B\pi_{\min}^2
        \gtrsim
        \log(|E|/\zeta),
\]
where \(\pi_{\min}\) is the minimum inlier fraction and \(\zeta\) is the target failure probability.
\section{Experiments}

\subsection{Implementation details}

All dataset experiments use the frozen code protocol on 11 ETH3D scenes~\citep{schops2017multi}. Rotations are synchronized by MPLS~\citep{shi2020message}; the view graph and world-frame correspondence normals are cached after this rotation stage. Unless stated otherwise, experiments use seeds 2026--2030, and \TriDE{} never constructs, thresholds, or prunes the graph. 

We report angular direction errors in degrees: mean $\bar e$, median $\tilde e$, and 90th percentile $e_{90}$; similarity-normalized location errors: mean \tMean{}, median \tMed{}, and 90th percentile \tP{} and logged direction-stage runtime \tauDir{}. For each upstream family, \dTauDir{} subtracts the matching plain upstream runtime. All runtime numbers were logged on an Apple M4 Pro machine with 24~GB RAM; they are used as implementation-level diagnostics rather than hardware-normalized compute comparisons. Dataset summaries use scene-equal macro averaging: average over seeds within each scene, then average scenes with equal weight. For the auxiliary location benchmark, aggregate errors average only valid scene-level entries rather than imposing an arbitrary penalty on solver failures.

The primary direction benchmark compares each upstream initializer (PCA, STE~\citep{yu2024subspace}, and FMS~\citep{lerman2018fast}) with three graph-level variants: diagnostic continuous \textsc{GN}, diagnostic continuous \textsc{LM}, and +\TriDE{}. The auxiliary location benchmark uses the same initializers and compares the plain upstream estimates, 1DSfM~\citep{wilson2014robust}, IR-AAB~\citep{shi2018estimation}, the same diagnostic \textsc{GN}/\textsc{LM} variants from Appendix~\ref{app:gn-lm}, and \TriDE{}, followed by CycleSync~\citep{licycle} or LUD~\citep{ozyesil2015robust,lerman2018exact}.  The frozen \TriDE{} configuration uses $\sigma=1^\circ$, $n_{\mathrm{cand}}=25$, $\beta=15$, nondegeneracy threshold $a_{\min}=10^{-3}$, at most $k_{\max}=4$ estimation sweeps, and stopping tolerance $10^{-3}$ degrees.

\subsection{Main direction benchmark}

\begin{table}[H]
\centering
\scriptsize
\setlength{\tabcolsep}{2.2pt}
\caption{Main direction-only ETH3D benchmark. Results are scene-equal macro averages; lower angular error is better. \tauDir{} is logged direction-stage runtime. GN/LM are diagnostic determinant-enforcement baselines, and \textsc{RAN} is a random diagnostic rather than a practical initializer. Reductions are computed for +TriDE relative to the corresponding initializer.}
\label{tab:main-direction}
\resizebox{\textwidth}{!}{%
\begin{tabular}{lcccc ccc ccc cccc ccc}
\toprule
& \multicolumn{4}{c}{Initializer} & \multicolumn{3}{c}{GN} & \multicolumn{3}{c}{LM} & \multicolumn{4}{c}{+ \TriDE{}} & \multicolumn{3}{c}{Reduction (\%)} \\
\cmidrule(lr){2-5}\cmidrule(lr){6-8}\cmidrule(lr){9-11}\cmidrule(lr){12-15}\cmidrule(lr){16-18}
Method & \eMean & \eMed & \eP & \tauDir & \eMean & \eMed & \eP & \eMean & \eMed & \eP & \eMean & \eMed & \eP & \tauDir & \eMean & \eMed & \eP \\
\midrule

PCA & 10.916 & 2.978 & 35.379 & 0.657 & 10.744 & 2.748 & 34.382 & 10.870 & 2.965 & 35.575 & \textbf{7.157} & \textbf{1.950} & \textbf{22.477} & 0.703 & 34.436 & 34.520 & 36.468 \\
STE & 8.904 & 1.891 & 30.608 & 0.665 & 8.936 & 2.257 & 30.045 & 8.882 & 1.887 & 30.550 & \textbf{6.678} & \textbf{1.819} & \textbf{20.170} & 0.706 & 24.997 & 3.802 & 34.102 \\
FMS & 9.942 & 2.164 & 34.604 & 0.666 & 9.863 & 2.246 & 33.903 & 9.936 & 2.166 & 34.697 & \textbf{6.823} & \textbf{1.855} & \textbf{20.423} & 0.709 & 31.372 & 14.279 & 40.981 \\
\textsc{RAN} & 57.577 & 60.686 & 84.190 & 0.000 & 57.532 & 60.919 & 84.289 & 55.242 & 57.543 & 82.226 & \textbf{8.594} & \textbf{2.853} & \textbf{25.691} & 0.087 & 85.074 & 95.299 & 69.484 \\
\bottomrule
\end{tabular}%
}
\end{table}

Table~\ref{tab:main-direction} reports scene-equal direction errors for the initializers, the \textsc{GN}/\textsc{LM} diagnostics, and +\TriDE{}. For the practical initializers, adding \TriDE{} improves all three aggregate metrics: PCA drops from $(10.916,2.978,35.379)$ to $(7.157,1.950,22.477)$; STE from $(8.904,1.891,30.608)$ to $(6.678,1.819,20.170)$; and FMS from $(9.942,2.164,34.604)$ to $(6.823,1.855,20.423)$. Their 90th-percentile reductions are $36.468\%$, $34.102\%$, and $40.981\%$, respectively. For STE, the median initializer error is already low, so the gain is most visible in the mean and tail metrics. The \textsc{GN} and \textsc{LM} diagnostics make only modest aggregate changes relative to the initializers and sometimes worsen individual metrics, whereas +\TriDE{} is the best aggregate entry for every practical initializer and reported direction metric. We additionally include \textsc{RAN}, a diagnostic that initializes both directions and initial weights randomly before applying the same \TriDE{} update. This diagnostic drops from $(57.577,60.686,84.190)$ to $(8.594,2.853,25.691)$, including a 69.484\% reduction in \eP; it is used only to probe recovery signal and is not treated as a practical standalone SfM initializer. Per-scene results are in Appendix Table~\ref{tab:direction-per-scene-compact}; the claim is aggregate, not uniform dominance on every scene.
\subsection{Auxiliary downstream location benchmark}

\begin{table}[t]
\centering
\scriptsize
\setlength{\tabcolsep}{3.0pt}
\renewcommand{\arraystretch}{0.95}
\setlength{\aboverulesep}{0.25ex}
\setlength{\belowrulesep}{0.25ex}
\setlength{\cmidrulesep}{0.15ex}
\caption{Auxiliary location benchmark under PCA, STE, and FMS upstream directions. Results use scene-equal aggregation; all-seed failures are omitted from valid-scene averages. Runtime is shared by CycleSync and LUD.}
\label{tab:location}
\begin{tabular}{l cc ccc ccc}
\toprule
& & & \multicolumn{3}{c}{CycleSync} & \multicolumn{3}{c}{LUD} \\
\cmidrule(lr){4-6}\cmidrule(lr){7-9}
Method & \tauDir & \dTauDir & \tMean & \tMed & \tP & \tMean & \tMed & \tP \\
\midrule
STE & 0.669 & -- & 0.100 & 0.037 & 0.246 & 0.193 & 0.106 & 0.417 \\
STE + 1DSfM & 1.381 & 0.712 & 0.174 & 0.118 & 0.368 & 0.183 & 0.088 & 0.450 \\
STE + IR-AAB & 1.303 & 0.634 & 0.097 & \textbf{0.030} & 0.296 & 0.152 & \textbf{0.068} & 0.387 \\
STE + GN & 0.718 & 0.049 & 0.119 & 0.046 & 0.303 & 0.236 & 0.130 & 0.537 \\
STE + LM & 1.020 & 0.351 & 0.102 & 0.037 & 0.260 & 0.196 & 0.111 & 0.429 \\
STE + \TriDE{} & 0.711 & \textbf{0.042} & \textbf{0.092} & 0.039 & \textbf{0.214} & \textbf{0.142} & 0.070 & \textbf{0.307} \\
\midrule
PCA & 0.660 & -- & 0.155 & 0.062 & 0.410 & 0.281 & 0.193 & 0.627 \\
PCA + 1DSfM & 1.358 & 0.697 & 0.181 & 0.063 & 0.500 & 0.330 & 0.231 & 0.743 \\
PCA + IR-AAB & 1.296 & 0.636 & 0.161 & 0.060 & 0.426 & 0.238 & 0.139 & 0.554 \\
PCA + GN & 0.708 & \textbf{0.048} & 0.188 & 0.066 & 0.488 & 0.287 & 0.188 & 0.623 \\
PCA + LM & 0.994 & 0.334 & 0.147 & 0.057 & 0.363 & 0.268 & 0.184 & 0.585 \\
PCA + \TriDE{} & 0.711 & 0.051 & \textbf{0.107} & \textbf{0.041} & \textbf{0.282} & \textbf{0.155} & \textbf{0.085} & \textbf{0.337} \\
\midrule
FMS & 0.675 & -- & 0.129 & 0.043 & 0.363 & 0.226 & 0.155 & 0.492 \\
FMS + 1DSfM & 1.378 & 0.703 & 0.159 & 0.048 & 0.429 & 0.272 & 0.180 & 0.698 \\
FMS + IR-AAB & 1.298 & 0.623 & 0.127 & 0.042 & 0.392 & 0.184 & 0.103 & 0.443 \\
FMS + GN & 0.722 & 0.047 & 0.160 & 0.054 & 0.485 & 0.233 & 0.148 & 0.529 \\
FMS + LM & 1.011 & 0.336 & 0.136 & 0.045 & 0.436 & 0.240 & 0.162 & 0.542 \\
FMS + \TriDE{} & 0.718 & \textbf{0.043} & \textbf{0.091} & \textbf{0.038} & \textbf{0.224} & \textbf{0.144} & \textbf{0.074} & \textbf{0.320} \\
\bottomrule
\end{tabular}
\end{table}

Table~\ref{tab:location} tests whether the direction estimates improve downstream location solvers in aggregate. In the valid-scene averages, +\TriDE{} gives the best mean and 90th-percentile location errors for every upstream block and for both CycleSync and LUD. For PCA and FMS it also gives the best aggregate medians. For STE, IR-AAB has marginally lower aggregate medians than +\TriDE{} for both CycleSync (0.030 vs. 0.039) and LUD (0.068 vs. 0.070), so the downstream claim is strongest for mean and tail errors rather than every median entry.

Relative to the corresponding plain upstream graph, \TriDE{} improves CycleSync \tMean{} from 0.100 to 0.092 for STE, 0.155 to 0.107 for PCA, and 0.129 to 0.091 for FMS\@. For LUD, it improves \tMean{} from 0.193 to 0.142 for STE, 0.281 to 0.155 for PCA, and 0.226 to 0.144 for FMS. The same aggregate tail-error pattern holds for \tP{}; median improvements hold for PCA and FMS, while the STE median differences are small and not uniformly best. At the scene level, the seed-level logs show one qualitative failure case: on \emph{delivery area}, PCA + 1DSfM and FMS + 1DSfM produced no valid CycleSync output across all seeds, whereas the corresponding +\TriDE{} graphs yielded valid CycleSync estimates. This is not a general claim about 1DSfM and is not used as a quantitative score, but it illustrates why preserving graph connectivity can be useful under this protocol. The \textsc{LM} diagnostic improves over \textsc{GN} on the aggregate CycleSync metrics but does not match +\TriDE{} on mean and tail errors, supporting the use of triangle determinants as bounded verification signals rather than hard global constraints. We therefore treat the location experiment as an auxiliary aggregate check rather than a claim of uniform scene-level dominance.

\paragraph{Runtime.}
\TriDE{} adds 0.042--0.051 seconds over the upstream direction stage, compared with about 0.70 seconds for 1DSfM, 0.63 seconds for IR-AAB, and 0.334--0.351 seconds for the \textsc{LM} diagnostic. Thus it has lower logged overhead than the deletion-based filters and damped \textsc{LM} in this protocol; the cheaper \textsc{GN} diagnostic is comparably fast but less accurate.

\begin{figure}[h]
    \centering
    \includegraphics[width=0.95\linewidth]{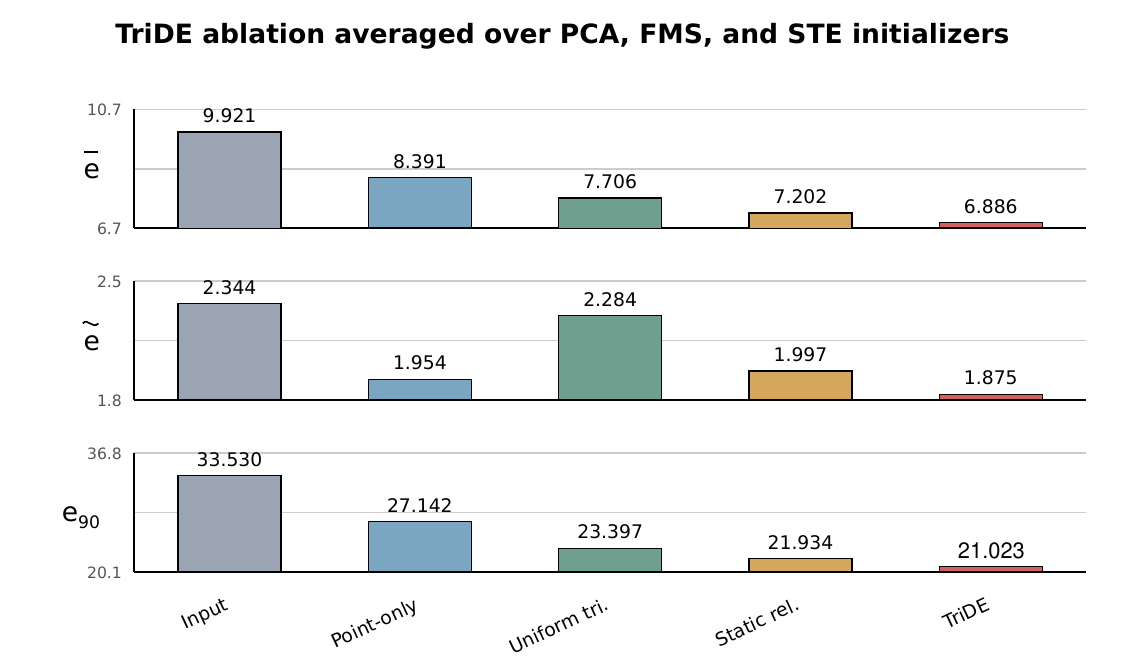}
    \caption{Ablation of candidate selection, triangle scoring, and measurement-based weighting. Local candidate selection gives a large initial gain; triangle scoring and reliability weighting further reduce mean and tail errors; dynamic measurement refresh gives the best aggregate values.}
    \label{fig:ablation}
\end{figure}
\subsection{Ablation}

The ablation studies in Figure~\ref{fig:ablation} compare \textsc{Input}, \textsc{Point-only}, \textsc{Uniform triangle}, \textsc{Static reliability}, and full \textsc{TriDE}. \textsc{Point-only} is the RANSAC-style edge-local ablation, using the same candidates but no triangle verification; \textsc{uniform triangle} adds coplanarity with equal triangle weights; \textsc{static reliability} fixes the initial point-measurement weights; full \TriDE{} refreshes reliability from the original point measurements after each sweep. The upper panel refers to the mean error; the middle panel refers to the median error; and the bottom panel refers to the 90th percentile error.

Averaged over PCA, FMS, and STE, errors improve from $(\eMean{},\eMed{},\eP{})=(9.921,2.344,33.530)$ for the input to $(6.886,1.875,21.023)$ for full \TriDE{}. The sequence shows that local candidate generation alone helps, triangle consistency further reduces mean and tail errors, reliability weighting helps further, and dynamic refresh gives the best overall profile.

\subsection{Synthetic Data Keypoint-corruption Stress Test}
\label{app:synthetic-corruption}
The synthetic data experiment in Figure \ref{fig:synthetic} uses 12 randomly generated cameras locations. For each edge, we randomly generate 80 keypoint matches per edge. We i.i.d. corrupt $x\%$ edges ($0\leq x\leq 70$). For each corrupted edge, we corrupt 80\% of its keypoint matches. Dashed curves in Figure~\ref{fig:synthetic} show the plain initializers, and solid curves apply \TriDE{} to the corresponding initializer. Shaded bands indicate $\pm 1$ standard deviation over random 5 seeds.

\begin{figure}[H]\label{fig:syn}
    \centering
    \includegraphics[width=1\linewidth]{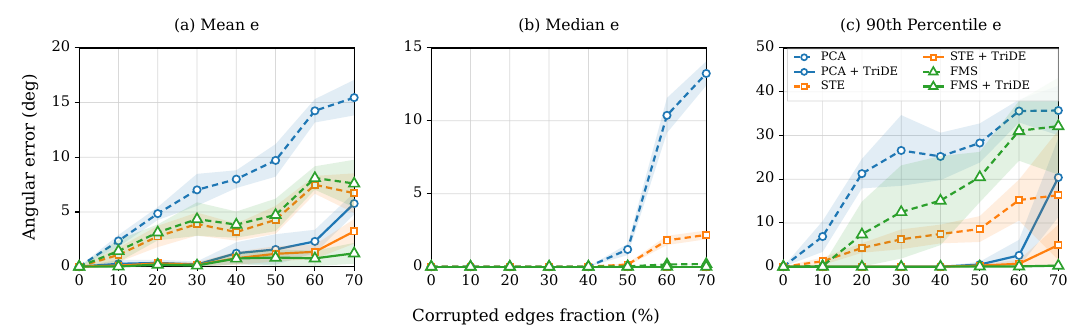}
    \caption{Synthetic keypoint-corruption stress test. \TriDE{} reduces error growth under corruption, especially for mean and 90th-percentile angular errors.}
    \label{fig:synthetic}
\end{figure}

The curves match the dataset trend: medians stay small at moderate corruption, while mean and 90th-percentile errors separate. At high corrupted-edge fractions, edge-local baselines develop large tails and \TriDE{} remains much lower. With $\leq 30\%$ of corruption, \TriDE{} successfully corrected almost all of the corrupted directions.

\section{Conclusion and Limitations}

We presented \TriDE{}, a graph-preserving message-passing estimator for pairwise translation directions.  The key idea is to jointly refine edge directions through triangle consistency while keeping the update efficient: each edge passes and receives reliability information through closed triples, selects candidates supported by both local correspondences and neighboring measurements, and refreshes its reliability from the original data.  This enables weak edges to be repaired once enough clean triangle witnesses and candidate recall are available.  Experiments show that \TriDE{} consistently improves direction estimation and generally improves aggregate downstream location recovery, especially in mean and tail error metrics.  Our method also has limitations.  It relies on sufficiently many closed triples, so very sparse or chain-like sensing regimes may require additional cues such as loop closures, multi-view tracks, temporal or inertial information, lines, or planes.  In addition, our current theory is idealized: it explains the observed phase transition under a structured recoverability model, but it does not yet cover fully adversarial corruptions where wrong directions can be arbitrary and potentially self-consistent.  Extending the analysis to such adversarial settings is an important direction for future work.

\section*{Acknowledgment}
This work is supported by the NSF award DMS-2514152.

\bibliographystyle{unsrtnat}
\bibliography{references}

@inproceedings{martinec2007robust,
  title={Robust rotation and translation estimation in multiview reconstruction},
  author={Martinec, Daniel and Pajdla, Tomas},
  booktitle={2007 IEEE conference on computer vision and pattern recognition},
  pages={1--8},
  year={2007},
  organization={IEEE}
}

@inproceedings{cui2015global,
  title={Global structure-from-motion by similarity averaging},
  author={Cui, Zhaopeng and Tan, Ping},
  booktitle={Proceedings of the IEEE international conference on computer vision},
  pages={864--872},
  year={2015}
}

@inproceedings{schonberger2016structure,
  title={Structure-from-motion revisited},
  author={Schonberger, Johannes L and Frahm, Jan-Michael},
  booktitle={Proceedings of the IEEE conference on computer vision and pattern recognition},
  pages={4104--4113},
  year={2016}
}

@inproceedings{pan2024global,
  title={Global structure-from-motion revisited},
  author={Pan, Linfei and Bar{\'a}th, D{\'a}niel and Pollefeys, Marc and Sch{\"o}nberger, Johannes L},
  booktitle={European Conference on Computer Vision},
  pages={58--77},
  year={2024},
  organization={Springer}
}

@book{hartley2003multiple,
  title={Multiple view geometry in computer vision},
  author={Hartley, Richard and Zisserman, Andrew},
  year={2003},
  publisher={Cambridge university press}
}

@article{hartley1997defense,
  title={In defense of the eight-point algorithm},
  author={Hartley, Richard I},
  journal={IEEE Transactions on pattern analysis and machine intelligence},
  volume={19},
  number={6},
  pages={580--593},
  year={1997},
  publisher={IEEE}
}

@article{nister2004efficient,
  title={An efficient solution to the five-point relative pose problem},
  author={Nist{\'e}r, David},
  journal={IEEE transactions on pattern analysis and machine intelligence},
  volume={26},
  number={6},
  pages={756--770},
  year={2004},
  publisher={IEEE}
}

@inproceedings{ozyesil2015robust,
  title={Robust camera location estimation by convex programming},
  author={Ozyesil, Onur and Singer, Amit},
  booktitle={Proceedings of the IEEE Conference on Computer Vision and Pattern Recognition},
  pages={2674--2683},
  year={2015}
}

@article{lerman2018fast,
  title={Fast, robust and non-convex subspace recovery},
  author={Lerman, Gilad and Maunu, Tyler},
  journal={Information and Inference: A Journal of the IMA},
  volume={7},
  number={2},
  pages={277--336},
  year={2018},
  publisher={Oxford University Press}
}

@inproceedings{yu2024subspace,
  title={A subspace-constrained Tyler's estimator and its applications to structure from motion},
  author={Yu, Feng and Zhang, Teng and Lerman, Gilad},
  booktitle={Proceedings of the IEEE/CVF Conference on Computer Vision and Pattern Recognition},
  pages={14575--14584},
  year={2024}
}

@inproceedings{wilson2014robust,
  title={Robust global translations with 1dsfm},
  author={Wilson, Kyle and Snavely, Noah},
  booktitle={European conference on computer vision},
  pages={61--75},
  year={2014},
  organization={Springer}
}

@inproceedings{shi2018estimation,
  title={Estimation of camera locations in highly corrupted scenarios: All about that base, no shape trouble},
  author={Shi, Yunpeng and Lerman, Gilad},
  booktitle={Proceedings of the IEEE Conference on Computer Vision and Pattern Recognition},
  pages={2868--2876},
  year={2018}
}

@inproceedings{shah2018view,
  title={View-graph selection framework for sfm},
  author={Shah, Rajvi and Chari, Visesh and Narayanan, PJ},
  booktitle={Proceedings of the European Conference on Computer Vision (ECCV)},
  pages={535--550},
  year={2018}
}

@inproceedings{manam2022correspondence,
  title={Correspondence reweighted translation averaging},
  author={Manam, Lalit and Govindu, Venu Madhav},
  booktitle={European Conference on Computer Vision},
  pages={56--72},
  year={2022},
  organization={Springer}
}

@article{fischler1981random,
  title={Random sample consensus: a paradigm for model fitting with applications to image analysis and automated cartography},
  author={Fischler, Martin A and Bolles, Robert C},
  journal={Communications of the ACM},
  volume={24},
  number={6},
  pages={381--395},
  year={1981},
  publisher={ACM New York, NY, USA}
}

@inproceedings{barath2019magsac,
  title={MAGSAC: Marginalizing sample consensus},
  author={Barath, Daniel and Matas, Jiri and Noskova, Jana},
  booktitle={Proceedings of the IEEE/CVF conference on computer vision and pattern recognition},
  pages={10197--10205},
  year={2019}
}

@inproceedings{barath2020magsac++,
  title={MAGSAC++, a fast, reliable and accurate robust estimator},
  author={Barath, Daniel and Noskova, Jana and Ivashechkin, Maksym and Matas, Jiri},
  booktitle={Proceedings of the IEEE/CVF conference on computer vision and pattern recognition},
  pages={1304--1312},
  year={2020}
}

@article{lerman2018exact,
  title={Exact camera location recovery by least unsquared deviations},
  author={Lerman, Gilad and Shi, Yunpeng and Zhang, Teng},
  journal={SIAM Journal on Imaging Sciences},
  volume={11},
  number={4},
  pages={2692--2721},
  year={2018},
  publisher={SIAM}
}

@inproceedings{licycle,
  title={Cycle-Sync: Robust Global Camera Pose Estimation through Enhanced Cycle-Consistent Synchronization},
  author={Li, Shaohan and Shi, Yunpeng and Lerman, Gilad},
  booktitle={The Thirty-ninth Annual Conference on Neural Information Processing Systems},
  year={2025}
}

@article{levenberg1944method,
  title={A method for the solution of certain non-linear problems in least squares},
  author={Levenberg, Kenneth},
  journal={Quarterly of applied mathematics},
  volume={2},
  number={2},
  pages={164--168},
  year={1944}
}

@article{marquardt1963algorithm,
  title={An algorithm for least-squares estimation of nonlinear parameters},
  author={Marquardt, Donald W},
  journal={Journal of the society for Industrial and Applied Mathematics},
  volume={11},
  number={2},
  pages={431--441},
  year={1963},
  publisher={SIAM}
}

@inproceedings{chatterjee2013efficient,
  title={Efficient and robust large-scale rotation averaging},
  author={Chatterjee, Avishek and Govindu, Venu Madhav},
  booktitle={Proceedings of the IEEE international conference on computer vision},
  pages={521--528},
  year={2013}
}

@inproceedings{shi2020message,
  title={Message passing least squares framework and its application to rotation synchronization},
  author={Shi, Yunpeng and Lerman, Gilad},
  booktitle={International conference on machine learning},
  pages={8796--8806},
  year={2020},
  organization={PMLR}
}

@inproceedings{schops2017multi,
  title={A multi-view stereo benchmark with high-resolution images and multi-camera videos},
  author={Schops, Thomas and Schonberger, Johannes L and Galliani, Silvano and Sattler, Torsten and Schindler, Konrad and Pollefeys, Marc and Geiger, Andreas},
  booktitle={Proceedings of the IEEE conference on computer vision and pattern recognition},
  pages={3260--3269},
  year={2017}
}

@article{lerman2022robust,
  title={Robust group synchronization via cycle-edge message passing},
  author={Lerman, Gilad and Shi, Yunpeng},
  journal={Foundations of Computational Mathematics},
  volume={22},
  number={6},
  pages={1665--1741},
  year={2022},
  publisher={Springer}
}

@inproceedings{shi2022robust,
  title={Robust group synchronization via quadratic programming},
  author={Shi, Yunpeng and Wyeth, Cole M and Lerman, Gilad},
  booktitle={International Conference on Machine Learning},
  pages={20095--20105},
  year={2022},
  organization={PMLR}
}

@inproceedings{zhuang2018baseline,
  title={Baseline desensitizing in translation averaging},
  author={Zhuang, Bingbing and Cheong, Loong-Fah and Lee, Gim Hee},
  booktitle={Proceedings of the IEEE Conference on Computer Vision and Pattern Recognition},
  pages={4539--4547},
  year={2018}
}

@inproceedings{manam2024fusing,
  title={Fusing directions and displacements in translation averaging},
  author={Manam, Lalit and Govindu, Venu Madhav},
  booktitle={2024 International Conference on 3D Vision (3DV)},
  pages={75--84},
  year={2024},
  organization={IEEE}
}
\clearpage
\appendix

\section{Diagnostic Gauss-Newton and Levenberg-Marquardt Implementations}
\label{app:gn-lm}

This appendix records the determinant-enforcement route used only for diagnostics. The triangle relation can be imposed directly through a nonlinear constrained objective,
\begin{equation}
    \min_{\{\mathbf{g}_e\in\Sphere\}_{e\in E}} \sum_{e\in E}\sum_r \psi\!\left(r_{e,r}(\mathbf{g}_e)\right)
    \quad\text{s.t.}\quad
    \det(\mathbf{g}_{ij},\mathbf{g}_{jk},\mathbf{g}_{ik})=0\quad\forall (i,j,k)\in\mathcal{T},
    \label{eq:appendix-global-nonconvex}
\end{equation}
where $\psi$ is a robust point-residual penalty. Because the variables lie on a product of spheres and the determinant constraints are multilinear, this global route is nonconvex and couples many edges at once. We therefore use it only for diagnostic baselines, while \TriDE{} performs bounded per-edge candidate selection.

The two diagnostics are a constrained Gauss--Newton projection, denoted \textsc{GN}, and a damped Levenberg--Marquardt least-squares variant, denoted \textsc{LM}. Both rebuild local tangent bases and retract updated directions after each step.

Let the current unit direction on edge $e$ be $\mathbf{g}_e\in\Sphere$. Choose an orthonormal tangent-basis matrix $\mathbf{U}_e\in\R^{3\times 2}$ with $\mathbf{U}_e^\top \mathbf{g}_e=\mathbf{0}$, and write the tangent-increment vector as $\mathbf{z}_e\in\R^2$. After any step, directions are retracted by normalization,
\begin{equation}
    \mathbf{g}_e^{+}=\frac{\mathbf{g}_e+\mathbf{U}_e\mathbf{z}_e}{\norm{\mathbf{g}_e+\mathbf{U}_e\mathbf{z}_e}} .
    \label{eq:appendix-retraction}
\end{equation}
Let $\mathcal{G}=\{\mathbf{g}_e\}_{e\in E}$ denote the current direction collection. For a graph triangle $\tau=(a,b,c)$, define the scalar determinant residual
\begin{equation}
    d_\tau(\mathcal{G})=\mathbf{g}_a^\top(\mathbf{g}_b\times \mathbf{g}_c)\in\R .
    \label{eq:appendix-det-residual}
\end{equation}
Its first-order tangent Jacobian has three nonzero row-matrix blocks in $\R^{1\times 2}$,
\begin{equation}
\begin{aligned}
    \mathbf{C}_{\tau,a} &= (\mathbf{g}_b\times \mathbf{g}_c)^\top \mathbf{U}_a,\\
    \mathbf{C}_{\tau,b} &= (\mathbf{g}_c\times \mathbf{g}_a)^\top \mathbf{U}_b,\\
    \mathbf{C}_{\tau,c} &= (\mathbf{g}_a\times \mathbf{g}_b)^\top \mathbf{U}_c .
\end{aligned}
\label{eq:appendix-det-jacobian}
\end{equation}
Stacking all valid triangle rows gives the linearization $\mathbf{d}(\mathcal{G}+\mathbf{U}\mathbf{z})\approx \mathbf{d}+\mathbf{C}\mathbf{z}$, where $\mathbf{U}\in\R^{3m\times 2m}$ is the block-diagonal tangent-basis matrix, $\mathbf{d}\in\R^{n_{\mathrm{row}}}$ is the stacked residual vector, $\mathbf{C}\in\R^{n_{\mathrm{row}}\times 2m}$ is the stacked Jacobian matrix, and $\mathbf{z}\in\R^{2m}$ is the stacked tangent-increment vector. A row is retained only when
\begin{equation}
    \gamma_\tau=\bigl(\norm{\mathbf{g}_b\times \mathbf{g}_c}\,\norm{\mathbf{g}_c\times \mathbf{g}_a}\,\norm{\mathbf{g}_a\times \mathbf{g}_b}\bigr)^{1/3}>a_{\min},
    \label{eq:appendix-degeneracy}
\end{equation}
which removes nearly degenerate determinant constraints.

\paragraph{Constrained Gauss--Newton.}
The \textsc{GN} diagnostic is the stripped determinant-only projection used to test the hard-constraint route. It solves the minimum-norm tangent correction satisfying the linearized determinant equations,
\begin{equation}
    \min_{\mathbf{z}\in\R^{2m}} \frac{1}{2}\norm{\mathbf{z}}^2
    \quad\text{s.t.}\quad \mathbf{C}\mathbf{z}+\mathbf{d}=\mathbf{0} .
    \label{eq:appendix-gn-minnorm}
\end{equation}
The corresponding KKT system is
\begin{equation}
    \begin{bmatrix}
        \mathbf{I} & \mathbf{C}^\top\\
        \mathbf{C} & -\rho \mathbf{I}
    \end{bmatrix}
    \begin{bmatrix}
        \mathbf{z}\\ \boldsymbol{\lambda}
    \end{bmatrix}
    =-
    \begin{bmatrix}
        \mathbf{0}\\ \mathbf{d}
    \end{bmatrix},
    \label{eq:appendix-gn-kkt}
\end{equation}
where $\mathbf{I}$ is an identity matrix of the required size, $\boldsymbol{\lambda}\in\R^{n_{\mathrm{row}}}$ is the multiplier vector, and a tiny scalar $\rho\ge 0$ regularizes redundant or nearly dependent constraint rows. The solution is applied through Eq.~\eqref{eq:appendix-retraction}, and the linearization is rebuilt at the next iterate.

\paragraph{Levenberg--Marquardt diagnostic.}
The \textsc{LM} diagnostic treats point orthogonality and triangle consistency as one weighted nonlinear least-squares problem~\citep{levenberg1944method,marquardt1963algorithm}. For each edge, the point set is the small subset of local correspondence normals with the smallest current $\abs{\mathbf{g}_e^\top \mathbf{x}_{e,r}}$. For such a selected normal, let
\begin{equation}
    p_{e,r}(\mathcal{G})=\mathbf{g}_e^\top \mathbf{x}_{e,r}\in\R,
    \qquad
    \mathbf{A}_{e,r}=\mathbf{x}_{e,r}^\top \mathbf{U}_e\in\R^{1\times 2} .
\end{equation}
Point rows use Cauchy-type robust weights of the form
\begin{equation}
    w_{e,r}=\frac{1}{\sigma_e^2}\left(1+\left(\frac{p_{e,r}}{c\sigma_e}\right)^2\right)^{-1},
    \label{eq:appendix-point-weight}
\end{equation}
where $\sigma_e$ is a robust residual scale for the selected edge measurements. For triangle rows, use $d_\tau$ and the row matrix $\mathbf{C}_\tau\in\R^{1\times 2m}$ obtained by inserting the nonzero blocks from Eqs.~\eqref{eq:appendix-det-residual}--\eqref{eq:appendix-det-jacobian}. The local quadratic model is
\begin{equation}
\begin{aligned}
    \min_{\mathbf{z}}\;&\frac{1}{2}\sum_{e,r} w_{e,r}\bigl(p_{e,r}+\mathbf{A}_{e,r}\mathbf{z}_e\bigr)^2 \\
    &+\frac{\lambda_{\mathrm{tri}}}{2}\sum_{\tau} \omega_\tau\bigl(d_\tau+\mathbf{C}_\tau \mathbf{z}\bigr)^2 .
\end{aligned}
\label{eq:appendix-lm-model}
\end{equation}
The triangle weight is computed from the three incident edge reliabilities. For this diagnostic only, let $s_e$ denote the mean of the small absolute point residuals on edge $e$; the reliability is
\begin{equation}
    q_e=\exp(-\beta s_e),
    \qquad
    \omega_{(a,b,c)}\propto q_aq_bq_c,
    \label{eq:appendix-tri-weight}
\end{equation}
followed by normalization across the retained triangle rows. Thus triangles supported by locally reliable edges receive larger weight.

Writing Eq.~\eqref{eq:appendix-lm-model} as $\frac12 \mathbf{z}^\top \mathbf{H}\mathbf{z}+\mathbf{b}^\top \mathbf{z}+\mathrm{const}$, with $\mathbf{H}\in\R^{2m\times 2m}$ and $\mathbf{b}\in\R^{2m}$, gives the damped LM step
\begin{equation}
    \left(\mathbf{H}+\mu\,\mathrm{diag}(\mathbf{H})\right)\mathbf{z}=-\mathbf{b} .
    \label{eq:appendix-lm-step}
\end{equation}
The damping $\mu$ is increased after rejected steps and decreased after accepted objective decreases. The accepted step is retracted by Eq.~\eqref{eq:appendix-retraction}. Thus \textsc{LM} is a continuous, coupled baseline, while \TriDE{} remains a bounded candidate-selection method.

\section{Hyperparameter Sensitivity}
\label{app:sensitivity}
We vary $n_{\mathrm{cand}}$, $\beta$, and $k_{\max}$ one at a time on the PCA direction benchmark, using the same scene-equal averaging as in the main text. The triangle-validity threshold remains fixed because it is a degeneracy safeguard rather than a performance knob. The corresponding sensitivity curves are shown in Figures~\ref{fig:sensitivity-B}--\ref{fig:sensitivity-K}.

\begin{figure}[H]
    \centering
    \includegraphics[width=0.98\linewidth]{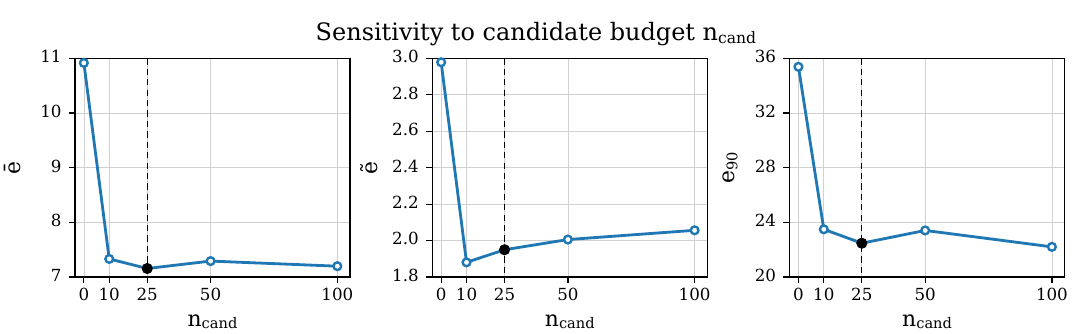}
    \caption{Sensitivity to candidate budget $n_{\mathrm{cand}}$. Accuracy improves from very small pools to moderate budgets and is stable around the default $n_{\mathrm{cand}}=25$.}
    \label{fig:sensitivity-B}
\end{figure}

\begin{figure}[H]
    \centering
    \includegraphics[width=0.98\linewidth]{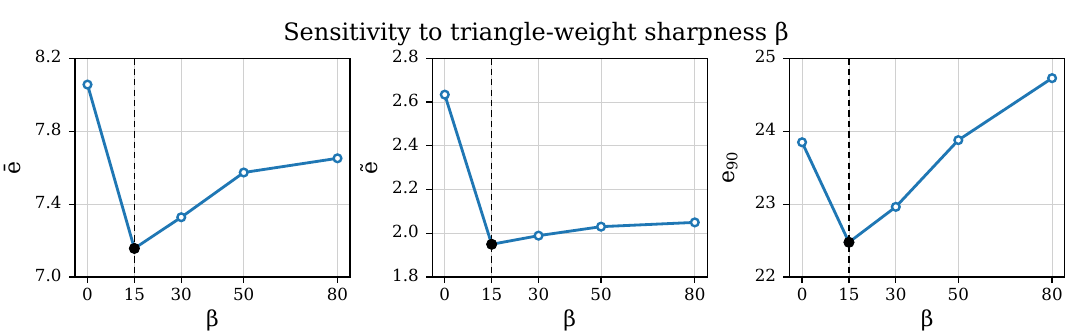}
    \caption{Sensitivity to triangle-weight sharpness $\beta$. The default $\beta=15$ lies in a broad plateau rather than at a narrow optimum.}
    \label{fig:sensitivity-beta}
\end{figure}

\begin{figure}[H]
    \centering
    \includegraphics[width=0.98\linewidth]{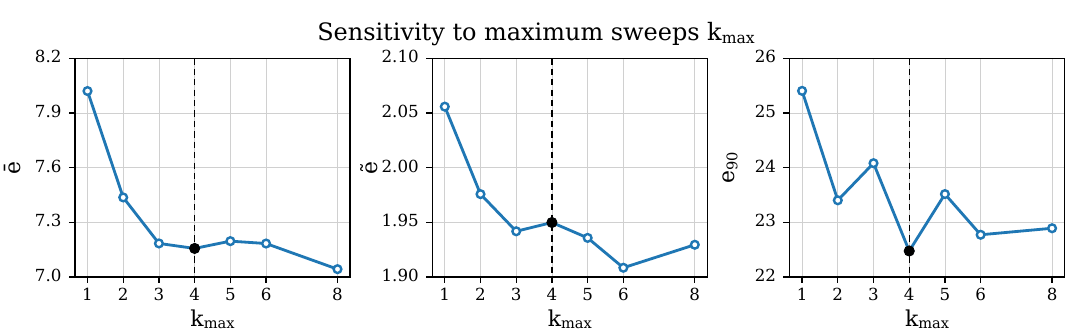}
    \caption{Sensitivity to estimation sweeps $k_{\max}$. Most gains appear in the first few sweeps, and the default $k_{\max}=4$ is on the stable plateau.}
    \label{fig:sensitivity-K}
\end{figure}

The three sensitivity studies support the same qualitative conclusion: moderate changes in the main hyperparameters preserve the direction-error profile, so the reported setting is not a finely tuned isolated optimum.

\section{View-graph and Triangle Statistics}
\label{app:graph-stats}
Table~\ref{tab:graph-stats} gives the graph sizes used in Section~\ref{sec:complexity}. Across all 11 scenes, the caches contain 2,228 edges and 9,425 triangles; the pooled median incident-triangle count is 11, and only one edge has no triangle.
\begin{table}[H]
\centering
\small
\setlength{\tabcolsep}{4.0pt}
\caption{View-graph and triangle-context statistics for the cached ETH3D scenes. $d_e$ is the number of incident graph triangles for an edge, and $N_e$ is the number of stored correspondence normals on that edge. The reported protocol uses the full cached triangle set.}
\label{tab:graph-stats}
\begin{tabular}{lrrrrrr}
\toprule
Scene & $|V|$ & $|E|$ & $|\mathcal{T}|$ & median $d_e$ & edges in triangles & median $N_e$ \\
\midrule
courtyard & 38 & 442 & 2977 & 23 & 100.0\% & 618.0 \\
delivery\_area & 39 & 256 & 799 & 9 & 100.0\% & 684.5 \\
electro & 39 & 314 & 1231 & 12 & 100.0\% & 409.0 \\
kicker & 30 & 251 & 1163 & 15 & 100.0\% & 530.0 \\
meadow & 14 & 55 & 94 & 5 & 98.2\% & 132.0 \\
office & 21 & 128 & 402 & 10 & 100.0\% & 115.5 \\
pipes & 14 & 60 & 124 & 6 & 100.0\% & 248.0 \\
relief & 13 & 78 & 286 & 11 & 100.0\% & 1698.0 \\
relief\_2 & 20 & 150 & 593 & 12 & 100.0\% & 403.0 \\
terrace & 23 & 131 & 357 & 9 & 100.0\% & 581.0 \\
terrains & 42 & 363 & 1399 & 11 & 100.0\% & 287.0 \\
\midrule
Total / pooled & 293 & 2228 & 9425 & 11 & 99.96\% & 423.5 \\
\bottomrule
\end{tabular}
\end{table}

\section{Full Per-scene Direction Results}
{\scriptsize
\begin{table}[H]
\centering
\scriptsize
\setlength{\tabcolsep}{1.45pt}
\renewcommand{\arraystretch}{0.92}
\caption{Compact per-scene direction results for Init, GN, LM, and \TriDE{}. Lower is better; within each scene, initializer, and metric, bold indicates the best value among the four entries. GN and LM are the diagnostic global-determinant variants from Appendix~\ref{app:gn-lm}. The \textsc{RAN} group abbreviates a diagnostic with random directions and random initial weights.}
\label{tab:direction-per-scene-compact}
\resizebox{\textwidth}{!}{%
\begin{tabular}{l|ccc|ccc|ccc|ccc||ccc|ccc|ccc|ccc}
\toprule
Scene & \multicolumn{12}{c||}{\textsc{RAN}} & \multicolumn{12}{c}{PCA} \\
\cmidrule(lr){2-13}\cmidrule(lr){14-25}
 & \multicolumn{3}{c|}{Init} & \multicolumn{3}{c|}{GN} & \multicolumn{3}{c|}{LM} & \multicolumn{3}{c||}{\TriDE{}} & \multicolumn{3}{c|}{Init} & \multicolumn{3}{c|}{GN} & \multicolumn{3}{c|}{LM} & \multicolumn{3}{c}{\TriDE{}} \\
 & \eMean & \eMed & \eP & \eMean & \eMed & \eP & \eMean & \eMed & \eP & \eMean & \eMed & \eP & \eMean & \eMed & \eP & \eMean & \eMed & \eP & \eMean & \eMed & \eP & \eMean & \eMed & \eP \\
\midrule
courtyard & 57.477 & 60.424 & 84.442 & 57.501 & 60.947 & 84.345 & 55.554 & 58.624 & 82.797 & \textbf{14.924} & \textbf{2.557} & \textbf{55.615} & 19.258 & 3.559 & 63.277 & 19.128 & 3.402 & 61.591 & 19.258 & 3.559 & 63.078 & \textbf{12.815} & \textbf{2.174} & \textbf{48.635} \\
delivery\_area & 58.344 & 61.406 & 84.364 & 58.317 & 61.275 & 84.646 & 55.214 & 57.545 & 81.950 & \textbf{10.685} & \textbf{2.473} & \textbf{35.067} & 11.780 & 2.444 & 39.032 & 11.709 & 2.364 & 39.034 & 11.811 & 2.354 & 39.515 & \textbf{10.023} & \textbf{1.375} & \textbf{33.919} \\
electro & 56.962 & 60.471 & 84.079 & 56.974 & 60.661 & 83.956 & 54.953 & 57.213 & 82.494 & \textbf{13.795} & \textbf{3.601} & \textbf{45.010} & 14.779 & 3.709 & 52.913 & 14.540 & 3.348 & 52.612 & 14.780 & 3.709 & 52.702 & \textbf{12.215} & \textbf{2.002} & \textbf{40.388} \\
kicker & 58.188 & 60.843 & 84.882 & 58.273 & 60.785 & 85.126 & 54.864 & 56.262 & 82.256 & \textbf{6.994} & \textbf{3.405} & \textbf{14.739} & 11.876 & 5.742 & 32.927 & 11.137 & 5.014 & 32.137 & 11.866 & 5.742 & 32.842 & \textbf{6.054} & \textbf{2.873} & \textbf{12.717} \\
meadow & 57.943 & 61.119 & 82.648 & 58.056 & 61.576 & 83.397 & 57.214 & 59.339 & 82.168 & \textbf{17.863} & \textbf{8.352} & \textbf{48.710} & 22.958 & 7.972 & 71.962 & 22.944 & 6.702 & 72.027 & 22.673 & 7.930 & 72.099 & \textbf{14.278} & \textbf{5.803} & \textbf{45.570} \\
office & 56.903 & 59.146 & 84.545 & 56.910 & 58.924 & 84.515 & 54.770 & 56.740 & 82.703 & \textbf{11.008} & \textbf{3.970} & \textbf{32.695} & 13.838 & 4.472 & 44.956 & 13.608 & 4.197 & 45.430 & 13.791 & 4.467 & 46.568 & \textbf{10.425} & \textbf{3.575} & \textbf{29.620} \\
pipes & 57.253 & 60.954 & 84.560 & 57.230 & 60.644 & 85.321 & 55.723 & 57.706 & 83.697 & \textbf{3.017} & \textbf{1.879} & \textbf{7.557} & 9.047 & 1.466 & 40.364 & 8.864 & 1.748 & 34.939 & 8.896 & 1.464 & 40.688 & \textbf{1.287} & \textbf{0.651} & \textbf{2.310} \\
relief & 57.399 & 60.044 & 84.320 & 56.604 & 60.222 & 83.601 & 53.180 & 55.038 & 80.044 & \textbf{0.624} & \textbf{0.376} & \textbf{1.333} & 0.228 & 0.205 & 0.448 & \textbf{0.210} & \textbf{0.194} & \textbf{0.389} & 0.229 & 0.211 & 0.445 & 0.229 & 0.206 & 0.426 \\
relief\_2 & 57.175 & 60.654 & 83.651 & 57.412 & 61.850 & 83.777 & 55.335 & 58.200 & 82.082 & \textbf{9.461} & \textbf{1.723} & \textbf{30.247} & 10.421 & 1.034 & 33.227 & 10.488 & 1.180 & 32.596 & 10.472 & 1.013 & 34.410 & \textbf{8.385} & \textbf{0.918} & \textbf{28.220} \\
terrace & 58.269 & 62.596 & 84.226 & 58.301 & 63.096 & 84.126 & 55.337 & 58.398 & 81.867 & \textbf{3.710} & \textbf{1.863} & \textbf{7.110} & 2.608 & 1.348 & 2.989 & 2.553 & \textbf{1.335} & \textbf{2.443} & 2.524 & 1.348 & 2.932 & \textbf{1.599} & 1.369 & 2.666 \\
terrains & 57.436 & 59.883 & 84.370 & 57.269 & 60.123 & 84.372 & 55.523 & 57.905 & 82.423 & \textbf{2.449} & \textbf{1.180} & \textbf{4.520} & 3.280 & 0.810 & 6.063 & 3.006 & 0.748 & 5.008 & 3.272 & 0.819 & 6.048 & \textbf{1.419} & \textbf{0.702} & \textbf{3.330} \\
\midrule
Scene & \multicolumn{12}{c||}{FMS} & \multicolumn{12}{c}{STE} \\
\cmidrule(lr){2-13}\cmidrule(lr){14-25}
 & \multicolumn{3}{c|}{Init} & \multicolumn{3}{c|}{GN} & \multicolumn{3}{c|}{LM} & \multicolumn{3}{c||}{\TriDE{}} & \multicolumn{3}{c|}{Init} & \multicolumn{3}{c|}{GN} & \multicolumn{3}{c|}{LM} & \multicolumn{3}{c}{\TriDE{}} \\
 & \eMean & \eMed & \eP & \eMean & \eMed & \eP & \eMean & \eMed & \eP & \eMean & \eMed & \eP & \eMean & \eMed & \eP & \eMean & \eMed & \eP & \eMean & \eMed & \eP & \eMean & \eMed & \eP \\
\midrule
courtyard & 18.582 & 2.632 & 62.733 & 18.473 & 2.391 & 62.729 & 18.575 & 2.616 & 62.769 & \textbf{12.677} & \textbf{2.133} & \textbf{48.997} & 16.643 & 2.171 & 59.674 & 16.548 & \textbf{2.136} & 59.040 & 16.640 & 2.171 & 59.513 & \textbf{12.472} & 2.137 & \textbf{48.349} \\
delivery\_area & 11.237 & 1.563 & 39.726 & 11.109 & 1.364 & 39.779 & 11.239 & 1.647 & 39.837 & \textbf{10.009} & \textbf{1.320} & \textbf{33.111} & 10.796 & 1.241 & 37.789 & 10.738 & \textbf{0.939} & 37.857 & 10.734 & 1.242 & 37.717 & \textbf{10.588} & 1.269 & \textbf{35.108} \\
electro & 14.267 & 2.574 & 60.709 & 14.205 & 2.492 & 60.429 & 14.268 & 2.559 & 61.459 & \textbf{12.129} & \textbf{1.859} & \textbf{40.065} & 13.497 & 1.869 & 50.791 & 13.363 & 2.025 & 50.846 & 13.480 & 1.869 & 50.704 & \textbf{12.309} & \textbf{1.761} & \textbf{42.571} \\
kicker & 8.720 & 4.098 & 22.263 & 7.995 & 3.723 & 18.689 & 8.693 & 4.135 & 21.617 & \textbf{6.099} & \textbf{2.785} & \textbf{12.699} & 8.134 & 3.388 & 19.954 & 7.418 & 2.953 & 17.793 & 8.145 & 3.386 & 19.992 & \textbf{5.913} & \textbf{2.748} & \textbf{12.612} \\
meadow & 18.576 & 5.262 & 68.416 & 18.461 & 5.351 & 68.344 & 18.591 & 5.183 & 68.342 & \textbf{11.221} & \textbf{5.118} & \textbf{23.675} & 11.687 & 5.004 & 48.258 & 11.014 & \textbf{4.189} & 48.266 & 11.650 & 4.965 & 47.974 & \textbf{9.340} & 4.895 & \textbf{16.600} \\
office & 12.183 & 3.756 & 38.599 & 12.228 & 4.139 & 38.488 & 12.266 & 3.752 & 38.633 & \textbf{10.271} & \textbf{3.549} & \textbf{28.448} & 10.675 & \textbf{3.475} & 36.080 & 11.006 & 3.962 & 35.220 & 10.714 & 3.479 & 36.089 & \textbf{10.238} & 3.478 & \textbf{27.820} \\
pipes & 8.765 & 0.849 & 35.850 & 9.055 & 1.476 & 35.682 & 8.653 & 0.865 & 35.938 & \textbf{1.206} & \textbf{0.669} & \textbf{2.289} & 4.709 & 0.630 & 3.373 & 4.922 & 0.939 & 3.381 & 4.580 & \textbf{0.626} & 3.367 & \textbf{0.946} & 0.648 & \textbf{2.363} \\
relief & 0.227 & 0.203 & 0.417 & \textbf{0.213} & \textbf{0.201} & \textbf{0.386} & 0.226 & 0.202 & 0.423 & 0.220 & 0.208 & 0.426 & 0.203 & 0.183 & 0.376 & \textbf{0.192} & \textbf{0.181} & \textbf{0.346} & 0.203 & 0.183 & 0.376 & 0.210 & 0.191 & 0.404 \\
relief\_2 & 11.569 & 0.895 & 45.852 & 11.826 & 1.591 & 43.344 & 11.566 & 0.894 & 45.941 & \textbf{8.174} & \textbf{0.849} & \textbf{26.638} & 17.138 & 0.864 & 73.972 & 18.985 & 5.532 & 72.856 & 17.134 & \textbf{0.861} & 73.881 & \textbf{8.456} & 0.864 & \textbf{30.051} \\
terrace & 2.595 & 1.319 & 2.723 & 2.568 & \textbf{1.267} & \textbf{2.333} & 2.578 & 1.319 & 2.670 & \textbf{1.642} & 1.364 & 2.740 & 2.486 & 1.319 & 2.679 & 2.450 & \textbf{1.243} & \textbf{2.387} & 2.439 & 1.318 & 2.692 & \textbf{1.584} & 1.339 & 2.637 \\
terrains & 2.645 & \textbf{0.654} & 4.156 & 2.356 & 0.717 & \textbf{2.728} & 2.636 & \textbf{0.654} & 4.034 & \textbf{1.410} & 0.688 & 3.354 & 1.977 & 0.660 & 3.742 & 1.663 & 0.730 & \textbf{2.497} & 1.982 & \textbf{0.659} & 3.743 & \textbf{1.406} & 0.684 & 3.353 \\
\bottomrule
\end{tabular}%
}
\end{table}

}
\clearpage
\ifdefined\trideIncludePerSceneLocation
\section{Full Per-scene Location Results}
{\scriptsize
\begingroup
\centering
\scriptsize
\setlength{\tabcolsep}{1.3pt}
\renewcommand{\arraystretch}{0.90}
\begin{longtable}{l|ccc|ccc|ccc|ccc|ccc|ccc}
\caption{Compact per-scene location results under LUD and CycleSync, shown as stacked solver panels. Results are averaged over seeds. Bold marks the best variant within each upstream family for each scene, solver, and metric. Dashes indicate all-seed failure.}\label{tab:location-per-scene-compact}\\
\toprule
\endfirsthead
\caption[]{Compact per-scene location results  (continued).}\\
\toprule
\endhead
\midrule
\multicolumn{19}{r}{\emph{continued on next page}}\\
\endfoot
\bottomrule
\endlastfoot
\multicolumn{19}{l}{\emph{STE upstream directions, LUD}} \\
\midrule
Scene & \multicolumn{3}{c|}{STE} & \multicolumn{3}{c|}{STE + 1DSfM} & \multicolumn{3}{c|}{STE + IR-AAB} & \multicolumn{3}{c|}{STE + GN} & \multicolumn{3}{c|}{STE + LM} & \multicolumn{3}{c}{STE + \TriDE{}} \\
   & \tMean & \tMed & \tP & \tMean & \tMed & \tP & \tMean & \tMed & \tP & \tMean & \tMed & \tP & \tMean & \tMed & \tP & \tMean & \tMed & \tP \\
\midrule
courtyard & 0.738 & 0.404 & 1.656 & 0.568 & 0.261 & 1.387 & \textbf{0.401} & \textbf{0.055} & \textbf{1.145} & 0.714 & 0.361 & 1.643 & 0.734 & 0.383 & 1.680 & 0.534 & 0.239 & 1.188 \\
delivery area & 0.301 & 0.186 & 0.513 & 0.340 & 0.219 & 0.672 & \textbf{0.145} & \textbf{0.044} & \textbf{0.281} & 0.314 & 0.205 & 0.437 & 0.327 & 0.208 & 0.566 & 0.196 & 0.108 & 0.299 \\
electro & 0.278 & 0.086 & 1.065 & 0.237 & 0.042 & 1.051 & \textbf{0.208} & \textbf{0.029} & 0.954 & 0.274 & 0.082 & 1.112 & 0.283 & 0.102 & 1.102 & 0.226 & 0.049 & \textbf{0.915} \\
kicker & 0.078 & 0.052 & 0.095 & 0.078 & 0.051 & 0.096 & \textbf{0.052} & \textbf{0.046} & 0.089 & 0.074 & 0.052 & 0.094 & 0.069 & 0.052 & 0.095 & 0.053 & 0.048 & \textbf{0.088} \\
meadow & 0.235 & 0.148 & 0.558 & 0.285 & \textbf{0.118} & 0.960 & 0.184 & 0.130 & 0.400 & 0.468 & 0.250 & 1.398 & 0.231 & 0.144 & 0.555 & \textbf{0.168} & 0.127 & \textbf{0.396} \\
office & 0.228 & 0.104 & 0.182 & 0.228 & 0.104 & 0.188 & 0.564 & 0.360 & 1.132 & 0.247 & 0.109 & 0.182 & 0.229 & 0.104 & 0.192 & \textbf{0.221} & \textbf{0.096} & \textbf{0.164} \\
pipes & 0.013 & 0.012 & 0.021 & 0.017 & 0.013 & 0.042 & \textbf{0.010} & \textbf{0.008} & \textbf{0.016} & 0.023 & 0.025 & 0.036 & 0.013 & 0.013 & 0.019 & 0.019 & 0.011 & 0.039 \\
relief & \textbf{0.003} & 0.003 & \textbf{0.008} & \textbf{0.003} & 0.003 & \textbf{0.008} & \textbf{0.003} & 0.003 & \textbf{0.008} & -- & -- & -- & \textbf{0.003} & 0.003 & \textbf{0.008} & \textbf{0.003} & \textbf{0.002} & \textbf{0.008} \\
relief 2 & 0.196 & 0.135 & 0.390 & 0.203 & 0.120 & 0.459 & \textbf{0.059} & \textbf{0.034} & \textbf{0.141} & 0.196 & 0.179 & 0.393 & 0.220 & 0.174 & 0.406 & 0.079 & 0.043 & 0.168 \\
terrace & \textbf{0.026} & 0.024 & 0.042 & \textbf{0.026} & 0.025 & 0.041 & \textbf{0.026} & 0.025 & 0.041 & \textbf{0.026} & \textbf{0.023} & \textbf{0.040} & \textbf{0.026} & 0.024 & 0.044 & 0.032 & 0.029 & 0.062 \\
terrains & \textbf{0.025} & \textbf{0.013} & 0.051 & \textbf{0.025} & \textbf{0.013} & 0.050 & \textbf{0.025} & \textbf{0.013} & 0.050 & \textbf{0.025} & 0.018 & \textbf{0.036} & \textbf{0.025} & \textbf{0.013} & 0.051 & 0.026 & \textbf{0.013} & 0.051 \\
\midrule
\multicolumn{19}{l}{\emph{STE upstream directions, CycleSync}} \\
\midrule
Scene & \multicolumn{3}{c|}{STE} & \multicolumn{3}{c|}{STE + 1DSfM} & \multicolumn{3}{c|}{STE + IR-AAB} & \multicolumn{3}{c|}{STE + GN} & \multicolumn{3}{c|}{STE + LM} & \multicolumn{3}{c}{STE + \TriDE{}} \\
   & \tMean & \tMed & \tP & \tMean & \tMed & \tP & \tMean & \tMed & \tP & \tMean & \tMed & \tP & \tMean & \tMed & \tP & \tMean & \tMed & \tP \\
\midrule
courtyard & 0.235 & 0.039 & 0.764 & 1.002 & 0.840 & 1.907 & 0.345 & 0.038 & 1.186 & 0.264 & 0.037 & 0.931 & 0.238 & 0.036 & 0.846 & \textbf{0.200} & \textbf{0.035} & \textbf{0.576} \\
delivery area & 0.151 & 0.048 & 0.292 & -- & -- & -- & \textbf{0.148} & 0.047 & \textbf{0.277} & 0.153 & \textbf{0.040} & 0.310 & 0.153 & 0.049 & 0.295 & 0.156 & 0.060 & 0.295 \\
electro & 0.208 & \textbf{0.031} & 0.961 & 0.209 & \textbf{0.031} & 0.957 & 0.206 & \textbf{0.031} & 0.954 & 0.208 & 0.032 & 0.971 & 0.207 & \textbf{0.031} & 0.961 & \textbf{0.200} & 0.034 & \textbf{0.862} \\
kicker & 0.052 & 0.046 & 0.089 & 0.052 & 0.046 & 0.089 & 0.051 & 0.046 & 0.089 & \textbf{0.047} & \textbf{0.040} & \textbf{0.082} & 0.051 & 0.046 & 0.090 & 0.053 & 0.047 & 0.089 \\
meadow & 0.142 & \textbf{0.080} & 0.280 & 0.164 & 0.098 & 0.398 & 0.142 & 0.080 & 0.280 & 0.132 & 0.114 & 0.230 & 0.159 & \textbf{0.080} & 0.347 & \textbf{0.078} & 0.083 & \textbf{0.136} \\
office & 0.221 & 0.097 & \textbf{0.154} & 0.221 & 0.097 & 0.156 & -- & -- & -- & 0.241 & 0.097 & 0.180 & 0.222 & 0.097 & 0.155 & \textbf{0.219} & \textbf{0.094} & 0.162 \\
pipes & \textbf{0.010} & \textbf{0.008} & 0.016 & \textbf{0.010} & \textbf{0.008} & 0.016 & \textbf{0.010} & \textbf{0.008} & \textbf{0.015} & 0.015 & 0.013 & 0.030 & \textbf{0.010} & \textbf{0.008} & 0.017 & 0.013 & 0.011 & 0.024 \\
relief & \textbf{0.003} & \textbf{0.002} & \textbf{0.008} & \textbf{0.003} & \textbf{0.002} & \textbf{0.008} & \textbf{0.003} & \textbf{0.002} & \textbf{0.008} & -- & -- & -- & \textbf{0.003} & \textbf{0.002} & \textbf{0.008} & \textbf{0.003} & 0.003 & \textbf{0.008} \\
relief 2 & 0.023 & \textbf{0.016} & 0.053 & 0.024 & \textbf{0.016} & 0.055 & 0.023 & 0.017 & 0.052 & 0.082 & 0.046 & 0.226 & \textbf{0.022} & 0.018 & \textbf{0.047} & 0.042 & 0.025 & 0.104 \\
terrace & \textbf{0.026} & 0.025 & 0.041 & \textbf{0.026} & 0.025 & 0.041 & \textbf{0.026} & 0.025 & 0.041 & \textbf{0.026} & \textbf{0.024} & \textbf{0.039} & \textbf{0.026} & 0.025 & 0.041 & 0.027 & 0.027 & 0.044 \\
terrains & 0.025 & \textbf{0.013} & 0.050 & 0.025 & \textbf{0.013} & 0.050 & 0.025 & \textbf{0.013} & 0.050 & \textbf{0.018} & 0.018 & \textbf{0.033} & 0.025 & \textbf{0.013} & 0.050 & 0.025 & \textbf{0.013} & 0.050 \\
\midrule
\multicolumn{19}{l}{\emph{PCA upstream directions, LUD}} \\
\midrule
Scene & \multicolumn{3}{c|}{PCA} & \multicolumn{3}{c|}{PCA + 1DSfM} & \multicolumn{3}{c|}{PCA + IR-AAB} & \multicolumn{3}{c|}{PCA + GN} & \multicolumn{3}{c|}{PCA + LM} & \multicolumn{3}{c}{PCA + \TriDE{}} \\
   & \tMean & \tMed & \tP & \tMean & \tMed & \tP & \tMean & \tMed & \tP & \tMean & \tMed & \tP & \tMean & \tMed & \tP & \tMean & \tMed & \tP \\
\midrule
courtyard & 0.789 & 0.625 & 1.454 & 0.750 & 0.484 & 1.599 & 0.757 & 0.342 & 1.848 & 0.812 & 0.596 & 1.480 & 0.789 & 0.625 & 1.454 & \textbf{0.551} & \textbf{0.262} & \textbf{1.200} \\
delivery area & 0.436 & 0.283 & 0.802 & 0.967 & 0.909 & 1.679 & 0.258 & 0.168 & 0.406 & 0.420 & 0.267 & 0.696 & 0.470 & 0.282 & 0.867 & \textbf{0.217} & \textbf{0.123} & \textbf{0.337} \\
electro & 0.411 & 0.187 & 1.178 & 0.367 & 0.145 & 1.114 & 0.378 & 0.167 & 1.137 & 0.415 & 0.220 & 1.273 & 0.402 & 0.189 & 1.150 & \textbf{0.243} & \textbf{0.058} & \textbf{0.957} \\
kicker & 0.193 & 0.103 & 0.534 & 0.228 & 0.107 & 0.559 & 0.179 & 0.109 & 0.445 & 0.192 & 0.095 & 0.454 & 0.181 & 0.100 & 0.436 & \textbf{0.053} & \textbf{0.046} & \textbf{0.091} \\
meadow & 0.618 & 0.489 & 1.423 & 0.666 & 0.449 & 1.660 & 0.525 & 0.433 & 1.088 & 0.655 & 0.447 & 1.475 & 0.486 & 0.404 & 1.133 & \textbf{0.252} & \textbf{0.211} & \textbf{0.521} \\
office & 0.316 & 0.154 & 0.940 & 0.304 & 0.152 & 0.942 & 0.266 & 0.130 & 0.729 & 0.317 & 0.149 & 0.926 & 0.309 & 0.158 & 0.843 & \textbf{0.208} & \textbf{0.102} & \textbf{0.247} \\
pipes & 0.098 & 0.097 & 0.185 & 0.125 & 0.114 & 0.239 & 0.068 & 0.049 & 0.153 & 0.108 & 0.097 & 0.193 & 0.088 & 0.084 & 0.171 & \textbf{0.027} & \textbf{0.013} & \textbf{0.075} \\
relief & \textbf{0.004} & 0.004 & 0.008 & \textbf{0.004} & 0.004 & 0.008 & 0.005 & 0.004 & 0.011 & \textbf{0.004} & 0.004 & \textbf{0.007} & \textbf{0.004} & 0.004 & 0.008 & \textbf{0.004} & \textbf{0.003} & 0.008 \\
relief 2 & 0.137 & 0.127 & 0.230 & 0.137 & 0.127 & 0.230 & 0.105 & 0.087 & \textbf{0.173} & 0.147 & 0.134 & 0.223 & 0.137 & 0.122 & 0.230 & \textbf{0.094} & \textbf{0.072} & 0.175 \\
terrace & 0.042 & 0.029 & 0.058 & 0.041 & 0.028 & 0.053 & 0.044 & 0.029 & 0.047 & 0.042 & \textbf{0.026} & 0.059 & 0.042 & 0.029 & 0.060 & \textbf{0.026} & \textbf{0.026} & \textbf{0.044} \\
terrains & 0.043 & 0.028 & 0.085 & 0.039 & 0.018 & 0.087 & 0.038 & 0.015 & 0.055 & 0.041 & 0.030 & 0.068 & 0.043 & 0.029 & 0.087 & \textbf{0.026} & \textbf{0.014} & \textbf{0.050} \\
\midrule
\multicolumn{19}{l}{\emph{PCA upstream directions, CycleSync}} \\
\midrule
Scene & \multicolumn{3}{c|}{PCA} & \multicolumn{3}{c|}{PCA + 1DSfM} & \multicolumn{3}{c|}{PCA + IR-AAB} & \multicolumn{3}{c|}{PCA + GN} & \multicolumn{3}{c|}{PCA + LM} & \multicolumn{3}{c}{PCA + \TriDE{}} \\
   & \tMean & \tMed & \tP & \tMean & \tMed & \tP & \tMean & \tMed & \tP & \tMean & \tMed & \tP & \tMean & \tMed & \tP & \tMean & \tMed & \tP \\
\midrule
courtyard & 0.417 & 0.050 & 1.251 & 0.535 & 0.053 & 1.573 & 0.451 & 0.051 & 1.500 & 0.442 & 0.049 & 1.370 & 0.417 & 0.050 & 1.251 & \textbf{0.244} & \textbf{0.038} & \textbf{0.988} \\
delivery area & 0.253 & 0.146 & 0.446 & -- & -- & -- & 0.257 & 0.157 & 0.412 & 0.362 & 0.136 & 0.912 & 0.233 & 0.130 & 0.331 & \textbf{0.142} & \textbf{0.045} & \textbf{0.271} \\
electro & 0.321 & 0.088 & 0.935 & 0.329 & 0.112 & 0.901 & 0.320 & 0.072 & \textbf{0.893} & 0.362 & 0.119 & 0.906 & 0.346 & 0.092 & 0.969 & \textbf{0.212} & \textbf{0.032} & 0.901 \\
kicker & 0.123 & 0.073 & 0.132 & 0.138 & 0.074 & 0.317 & 0.127 & 0.061 & 0.153 & 0.104 & 0.060 & 0.118 & 0.114 & 0.065 & 0.105 & \textbf{0.058} & \textbf{0.046} & \textbf{0.090} \\
meadow & 0.256 & 0.116 & 0.883 & 0.468 & 0.187 & 1.359 & 0.254 & 0.113 & 0.867 & 0.466 & 0.161 & 1.295 & \textbf{0.178} & \textbf{0.079} & 0.538 & 0.187 & 0.112 & \textbf{0.430} \\
office & 0.194 & 0.099 & 0.609 & 0.196 & 0.101 & 0.605 & 0.203 & 0.101 & 0.596 & 0.188 & \textbf{0.089} & 0.520 & \textbf{0.180} & 0.100 & 0.549 & 0.231 & 0.097 & \textbf{0.217} \\
pipes & 0.045 & 0.037 & 0.094 & 0.044 & 0.030 & 0.079 & 0.045 & 0.037 & 0.093 & 0.053 & 0.042 & 0.090 & 0.040 & 0.033 & 0.065 & \textbf{0.012} & \textbf{0.012} & \textbf{0.022} \\
relief & \textbf{0.004} & \textbf{0.004} & 0.008 & \textbf{0.004} & \textbf{0.004} & 0.008 & 0.005 & \textbf{0.004} & 0.011 & \textbf{0.004} & \textbf{0.004} & \textbf{0.007} & \textbf{0.004} & \textbf{0.004} & 0.008 & \textbf{0.004} & \textbf{0.004} & 0.009 \\
relief 2 & 0.028 & 0.021 & 0.060 & \textbf{0.027} & 0.021 & 0.059 & \textbf{0.027} & 0.021 & \textbf{0.055} & 0.030 & \textbf{0.020} & 0.060 & 0.037 & 0.031 & 0.078 & 0.037 & 0.024 & 0.077 \\
terrace & 0.029 & 0.028 & 0.042 & 0.029 & 0.028 & 0.042 & 0.043 & 0.027 & 0.047 & 0.029 & 0.027 & \textbf{0.041} & 0.029 & 0.027 & 0.043 & \textbf{0.026} & \textbf{0.025} & 0.043 \\
terrains & 0.037 & 0.018 & 0.054 & 0.037 & 0.018 & 0.053 & 0.039 & 0.017 & 0.063 & 0.032 & 0.022 & \textbf{0.046} & 0.038 & 0.018 & 0.054 & \textbf{0.025} & \textbf{0.014} & 0.050 \\
\midrule
\pagebreak[4]
\multicolumn{19}{l}{\emph{FMS upstream directions, LUD}} \\
\midrule
Scene & \multicolumn{3}{c|}{FMS} & \multicolumn{3}{c|}{FMS + 1DSfM} & \multicolumn{3}{c|}{FMS + IR-AAB} & \multicolumn{3}{c|}{FMS + GN} & \multicolumn{3}{c|}{FMS + LM} & \multicolumn{3}{c}{FMS + \TriDE{}} \\
   & \tMean & \tMed & \tP & \tMean & \tMed & \tP & \tMean & \tMed & \tP & \tMean & \tMed & \tP & \tMean & \tMed & \tP & \tMean & \tMed & \tP \\
\midrule
courtyard & 0.797 & 0.627 & 1.450 & 0.769 & 0.581 & 1.615 & 0.739 & \textbf{0.278} & 1.794 & 0.765 & 0.518 & 1.384 & 0.760 & 0.509 & 1.425 & \textbf{0.559} & 0.287 & \textbf{1.258} \\
delivery area & 0.355 & 0.223 & 0.593 & 0.717 & 0.633 & 1.256 & \textbf{0.181} & \textbf{0.083} & \textbf{0.295} & 0.296 & 0.202 & 0.448 & 0.373 & 0.234 & 0.629 & 0.192 & 0.097 & 0.329 \\
electro & 0.311 & 0.105 & 1.201 & 0.275 & 0.069 & 1.180 & 0.271 & 0.071 & 1.010 & 0.302 & 0.092 & 1.250 & 0.320 & 0.142 & 1.099 & \textbf{0.233} & \textbf{0.046} & \textbf{0.944} \\
kicker & 0.099 & 0.073 & 0.131 & 0.131 & 0.073 & 0.115 & 0.092 & 0.051 & 0.146 & 0.086 & 0.068 & 0.132 & 0.104 & 0.073 & 0.175 & \textbf{0.053} & \textbf{0.047} & \textbf{0.092} \\
meadow & 0.453 & 0.331 & 0.905 & 0.631 & 0.252 & 2.371 & 0.377 & 0.360 & 0.864 & 0.601 & 0.365 & 1.418 & 0.580 & 0.473 & 1.291 & \textbf{0.191} & \textbf{0.140} & \textbf{0.422} \\
office & 0.203 & 0.113 & 0.626 & 0.180 & 0.110 & 0.609 & \textbf{0.149} & 0.103 & 0.369 & 0.208 & 0.106 & 0.640 & 0.226 & 0.113 & 0.814 & 0.213 & \textbf{0.096} & \textbf{0.174} \\
pipes & 0.071 & 0.047 & 0.186 & 0.093 & 0.072 & 0.220 & 0.066 & 0.059 & 0.138 & 0.090 & 0.078 & 0.179 & 0.085 & 0.055 & 0.215 & \textbf{0.014} & \textbf{0.012} & \textbf{0.026} \\
relief & 0.005 & 0.005 & 0.011 & 0.005 & 0.005 & 0.011 & 0.007 & 0.004 & 0.018 & 0.005 & 0.005 & 0.010 & 0.005 & 0.005 & 0.011 & \textbf{0.004} & \textbf{0.003} & \textbf{0.008} \\
relief 2 & 0.136 & 0.144 & 0.201 & 0.135 & 0.143 & 0.198 & 0.091 & 0.080 & \textbf{0.150} & 0.154 & 0.147 & 0.276 & 0.132 & 0.135 & 0.202 & \textbf{0.075} & \textbf{0.043} & 0.173 \\
terrace & 0.028 & \textbf{0.024} & 0.050 & \textbf{0.027} & 0.025 & 0.045 & \textbf{0.027} & 0.026 & \textbf{0.042} & \textbf{0.027} & \textbf{0.024} & 0.045 & \textbf{0.027} & \textbf{0.024} & 0.050 & \textbf{0.027} & 0.026 & 0.046 \\
terrains & 0.030 & 0.014 & 0.055 & 0.028 & \textbf{0.013} & 0.054 & 0.026 & 0.014 & 0.051 & \textbf{0.024} & 0.019 & \textbf{0.042} & 0.030 & 0.014 & 0.054 & 0.026 & 0.014 & 0.051 \\
\midrule
\multicolumn{19}{l}{\emph{FMS upstream directions, CycleSync}} \\
\midrule
Scene & \multicolumn{3}{c|}{FMS} & \multicolumn{3}{c|}{FMS + 1DSfM} & \multicolumn{3}{c|}{FMS + IR-AAB} & \multicolumn{3}{c|}{FMS + GN} & \multicolumn{3}{c|}{FMS + LM} & \multicolumn{3}{c}{FMS + \TriDE{}} \\
   & \tMean & \tMed & \tP & \tMean & \tMed & \tP & \tMean & \tMed & \tP & \tMean & \tMed & \tP & \tMean & \tMed & \tP & \tMean & \tMed & \tP \\
\midrule
courtyard & 0.537 & 0.049 & 1.541 & 0.561 & 0.056 & 1.439 & 0.509 & 0.045 & 1.843 & 0.551 & 0.046 & 2.232 & 0.541 & 0.045 & 2.163 & \textbf{0.191} & \textbf{0.035} & \textbf{0.587} \\
delivery area & 0.158 & 0.065 & 0.289 & -- & -- & -- & 0.157 & 0.063 & 0.290 & 0.309 & 0.148 & 0.747 & 0.157 & 0.064 & 0.288 & \textbf{0.136} & \textbf{0.041} & \textbf{0.279} \\
electro & 0.238 & 0.046 & 0.979 & 0.238 & 0.046 & 0.974 & 0.230 & 0.050 & 0.978 & 0.234 & 0.044 & 0.977 & 0.230 & 0.049 & 0.969 & \textbf{0.197} & \textbf{0.031} & \textbf{0.866} \\
kicker & 0.075 & 0.047 & 0.103 & 0.079 & 0.049 & 0.112 & 0.082 & 0.048 & 0.103 & 0.070 & \textbf{0.042} & 0.095 & 0.075 & 0.047 & 0.103 & \textbf{0.051} & 0.046 & \textbf{0.090} \\
meadow & 0.177 & 0.078 & 0.565 & 0.469 & 0.143 & 1.243 & 0.179 & \textbf{0.077} & 0.582 & 0.333 & 0.107 & 0.754 & 0.218 & 0.106 & 0.754 & \textbf{0.109} & 0.080 & \textbf{0.249} \\
office & \textbf{0.140} & 0.104 & 0.331 & \textbf{0.140} & 0.104 & 0.331 & \textbf{0.140} & 0.104 & 0.333 & 0.151 & 0.107 & 0.340 & 0.179 & 0.104 & 0.329 & 0.218 & \textbf{0.101} & \textbf{0.214} \\
pipes & 0.018 & 0.018 & 0.028 & 0.020 & 0.019 & 0.033 & 0.017 & 0.018 & 0.027 & 0.035 & 0.031 & 0.062 & 0.019 & 0.017 & 0.038 & \textbf{0.013} & \textbf{0.012} & \textbf{0.023} \\
relief & 0.005 & 0.006 & 0.012 & 0.005 & 0.006 & 0.012 & 0.007 & 0.004 & 0.018 & 0.005 & 0.005 & 0.012 & 0.005 & 0.006 & 0.012 & \textbf{0.004} & \textbf{0.003} & \textbf{0.008} \\
relief 2 & \textbf{0.023} & \textbf{0.018} & 0.047 & \textbf{0.023} & \textbf{0.018} & 0.047 & \textbf{0.023} & \textbf{0.018} & 0.047 & 0.024 & 0.023 & \textbf{0.038} & 0.024 & \textbf{0.018} & 0.049 & 0.031 & 0.024 & 0.051 \\
terrace & 0.027 & 0.025 & 0.043 & 0.027 & 0.025 & 0.043 & 0.027 & 0.026 & 0.043 & 0.027 & \textbf{0.023} & \textbf{0.042} & \textbf{0.026} & 0.025 & 0.043 & 0.027 & 0.026 & 0.044 \\
terrains & 0.026 & \textbf{0.014} & 0.051 & 0.026 & \textbf{0.014} & 0.051 & 0.026 & \textbf{0.014} & 0.051 & \textbf{0.019} & 0.018 & \textbf{0.036} & 0.025 & \textbf{0.014} & 0.051 & 0.025 & \textbf{0.014} & 0.050 \\
\end{longtable}
\endgroup

}
\fi
\section{Additional Ablation by Initializer}
\begin{table}[!htbp]
\centering
\scriptsize
\setlength{\tabcolsep}{3.0pt}
\caption{Ablation broken down by direction initializer. Aggregate rows use the same scene-equal macro convention as the main direction table. Lower is better. Bold indicates the best value within each initializer and metric.}
\label{tab:ablation-by-init}
\resizebox{\columnwidth}{!}{%
\begin{tabular}{lccccccccc}
\toprule
& \multicolumn{3}{c}{PCA} & \multicolumn{3}{c}{FMS} & \multicolumn{3}{c}{STE} \\
\cmidrule(lr){2-4}\cmidrule(lr){5-7}\cmidrule(lr){8-10}
Variant & \eMean & \eMed & \eP & \eMean & \eMed & \eP & \eMean & \eMed & \eP \\
\midrule
Input only            & 10.916 & 2.978 & 35.379 &  9.942 & 2.164 & 34.604 &  8.904 & 1.891 & 30.608 \\
Point-only selection  &  8.389 & \textbf{1.950} & 27.168 &  8.393 & 1.955 & 27.172 &  8.392 & 1.956 & 27.086 \\
Uniform triangle score&  8.056 & 2.634 & 23.849 &  7.772 & 2.203 & 23.817 &  7.289 & 2.014 & 22.524 \\
Static reliability    &  7.618 & 2.176 & 23.351 &  7.258 & 1.977 & 22.229 &  6.731 & 1.837 & 20.221 \\
\TriDE{}              & \textbf{7.157} & \textbf{1.950} & \textbf{22.477} & \textbf{6.823} & \textbf{1.855} & \textbf{20.423} & \textbf{6.678} & \textbf{1.819} & \textbf{20.170} \\
\bottomrule
\end{tabular}%
}
\end{table}

\section{Candidate-pool Recall Diagnostic}
\label{app:candidate-recall}

For each cached ETH3D edge and seed, we sample only the random two-normal hypotheses and measure the best angular error to ground truth, excluding the retained initializer direction. With $n_{\mathrm{cand}}=25$, the random pool contains a candidate within $2^\circ$ in 67.2\% of scene-equal edge--seed cases and within $5^\circ$ in 84.7\%, with median best-pool error $1.305^\circ$. The per-scene breakdown is reported in Table~\ref{tab:candidate-recall}.

\begin{table}[H]
\centering
\small
\setlength{\tabcolsep}{5.0pt}
\caption{Candidate-pool recall diagnostic on the cached ETH3D scenes. The initializer is excluded: each entry evaluates only the best of $B$ random two-normal candidates against the ground-truth direction. Fractions report the percentage of scene-equal edge--seed cases whose best sampled candidate is within the stated angular threshold. Lower is better for angular errors.}
\label{tab:candidate-recall}
\begin{tabular}{rcccccc}
\toprule
$B$ & $\leq 1^\circ$ & $\leq 2^\circ$ & $\leq 5^\circ$ & Median & P90 & Mean \\
\midrule
5   & 42.0\% & 59.9\% & 77.5\% & 1.733 & 14.955 & 5.497 \\
10  & 46.1\% & 63.4\% & 81.5\% & 1.523 & 11.027 & 4.269 \\
25  & 51.2\% & 67.2\% & 84.7\% & 1.305 & 8.534  & 3.359 \\
50  & 55.3\% & 69.9\% & 86.9\% & 1.189 & 6.929  & 2.817 \\
100 & 58.4\% & 72.0\% & 88.8\% & 1.083 & 6.069  & 2.476 \\
\bottomrule
\end{tabular}
\end{table}

\section{Random-Initialization Diagnostic}
\label{app:random_init}
This synthetic diagnostic tests whether the triangle update contains recovery signal beyond relying on a good initializer; it is not part of the main benchmark and does not recast \TriDE{} as a standalone solver. We compare standard \textsc{STE+TriDE} with \textsc{RAN}, which randomizes both directions and weights before the same refresh procedure.

\begin{figure}[H]
    \centering
    \includegraphics[width=\linewidth]{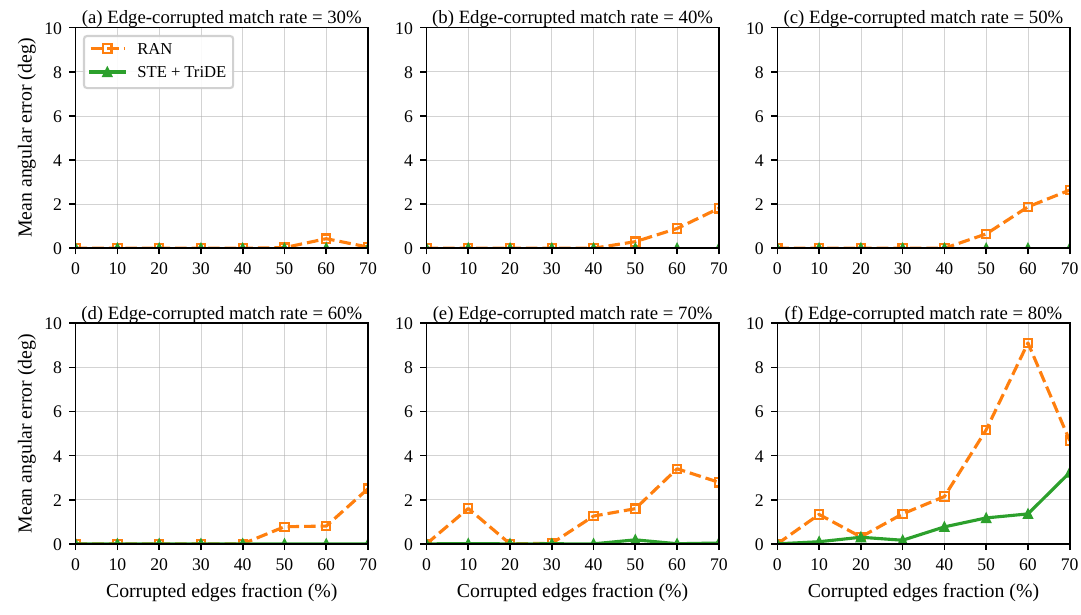}
    \caption{Random-initialization diagnostic on synthetic graphs. The standard initialized pipeline is most stable, while \textsc{RAN} can still recover in several moderate-corruption settings.}
    \label{fig:random_init}
\end{figure}

Figure~\ref{fig:random_init} compares the standard initialized pipeline against the random-direction, random-weight variant. A meaningful STE initializer provides the most reliable starting point, but the success of \textsc{RAN} in moderate regimes suggests that triangle consistency itself carries recovery information.
\clearpage

\providecommand{\TriDE}{TriDE}
\providecommand{\R}{\mathbb{R}}
\providecommand{\Sph}{\mathbb{S}}
\providecommand{\Prob}{\mathbb{P}}
\providecommand{\Exp}{\mathbb{E}}
\providecommand{\err}{\operatorname{err}}
\providecommand{\spanop}{\operatorname{span}}
\providecommand{\argmin}{\operatorname*{arg\,min}}

\section{Theory: Exact Direction Recovery for \TriDE{}}
\label{app:theory}

This appendix gives a streamlined proof of the exact-recovery mechanism used by
\TriDE{}.  The result is a direction-refinement theorem: it proves recovery of
unoriented pairwise translation directions after a single sufficiently sharp
synchronous triangle-weighted sweep.  It is not, by itself, a theorem for camera
locations.  Exact location recovery additionally requires a consistent choice of
signs and the usual bearing-rigidity or clean-location-solver assumptions.

The proof is organized as follows.  Section~\ref{app:theory-model} states the
one-sweep mathematical model.  Section~\ref{app:theory-det} proves the
main deterministic theorem.  Section~\ref{app:theory-prob-lemmas} records the
probabilistic ingredients needed to verify the deterministic assumptions.
Sections~\ref{app:theory-er} and~\ref{app:theory-rgg} give the Erd\H{o}s--R\'enyi
and random-geometric graph phase transitions.

\subsection{One-sweep model and notation}
\label{app:theory-model}

We analyze one synchronous \TriDE{} sweep.  The directions and point-support
scores before the sweep are denoted by
\(\{g_e^{(0)}\}_{e\in E}\) and \(\{A_e^{(0)}\}_{e\in E}\).  The same statement
applies to any later sweep after replacing the superscript \((0)\) by the
current sweep index.

\paragraph{Unoriented directions and error.}
For an observed edge \(e=ij\), let \(g_e^\star=g_{ij}^\star\in\Sph^2/\{\pm1\}\)
be the true unoriented direction.  In formulas we choose an arbitrary unit-vector
representative of \(g_e^\star\); all quantities below are sign-invariant.  For
\(c\in\Sph^2\), define
\begin{equation}
        \err(c,g_e^\star)
        :=
        \|P_{(g_e^\star)^\perp}c\|_2
        =
        \sqrt{1-(c^\top g_e^\star)^2}.
        \label{eq:theory-error}
\end{equation}
Thus \(\err(c,g_e^\star)=0\) if and only if \(c=\pm g_e^\star\).

\paragraph{Valid triangles.}
For a target edge \(ij\), let \(\mathcal N_{ij}\) be the set of retained triangles used by the update for \(ij\).  Thus \(k\in\mathcal N_{ij}\) only if
\(ik,jk\in E\) and the current cross product
\(g_{ik}^{(0)}\times g_{jk}^{(0)}\) passes the nondegeneracy test used by the
algorithm.  For such a retained triangle define
\begin{equation*}
        \widehat n_{ij,k}
        :=
        \frac{g_{ik}^{(0)}\times g_{jk}^{(0)}}
             {\|g_{ik}^{(0)}\times g_{jk}^{(0)}\|_2}.
\end{equation*}
For a nondegenerate true camera triangle, define the true triangle normal
\begin{equation*}
        n_{ij,k}^\star
        :=
        \frac{g_{ik}^\star\times g_{jk}^\star}
             {\|g_{ik}^\star\times g_{jk}^\star\|_2}.
\end{equation*}
The noiseless triangle relation is
\begin{equation*}
        (g_{ij}^\star)^\top n_{ij,k}^\star=0.
\end{equation*}
The sign of \(n_{ij,k}^\star\) is immaterial because the score uses absolute
inner products.

\paragraph{Point support and triangle weights.}
The point-support scale is fixed at
\begin{equation*}
        \sigma_0=1^\circ=\frac{\pi}{180}.
\end{equation*}
For an edge \(e\) with correspondence normals \(x_{e,1},\ldots,x_{e,N_e}\), define
\begin{equation}
        A_e(g)
        :=
        \frac1{N_e}\sum_{r=1}^{N_e}
        \exp\!\left(
        -\frac{\arcsin^2(|g^\top x_{e,r}|)}{2\sigma_0^2}
        \right),
        \qquad
        A_e^{(0)}:=A_e(g_e^{(0)}),
        \label{eq:theory-point-support}
\end{equation}
and \(s_e(g)=1-A_e(g)\).  For a target edge \(ij\) and witness \(k\), the two
neighboring edges are \(ik\) and \(jk\).  Define the support sum
\begin{equation*}
        H_{ij,k}:=A_{ik}^{(0)}+A_{jk}^{(0)}.
\end{equation*}
Since
\begin{equation*}
        \exp[-\beta(s_{ik}^{(0)}+s_{jk}^{(0)})]
        =
        e^{-2\beta}e^{\beta H_{ij,k}},
\end{equation*}
the normalized triangle weights are the softmax weights
\begin{equation}
        w_{ij,k}
        =
        \frac{e^{\beta H_{ij,k}}}
             {\sum_{\ell\in\mathcal N_{ij}}e^{\beta H_{ij,\ell}}}.
        \label{eq:triangle-softmax}
\end{equation}
For a candidate direction \(c\in\Sph^2\), the score minimized by the update is
\begin{equation*}
        Q_{ij,\beta}(c)
        =
        \sum_{k\in\mathcal N_{ij}}w_{ij,k}|c^\top \widehat n_{ij,k}|.
\end{equation*}
The output of one sweep is
\begin{equation*}
        g_{ij}^+
        \in
        \argmin_{c\in\mathcal C_{ij}}Q_{ij,\beta}(c),
\end{equation*}
where \(\mathcal C_{ij}\) is the finite candidate pool for edge \(ij\).  Because
the denominator in~\eqref{eq:triangle-softmax} is independent of \(c\), the same
candidate minimizes the unnormalized score
\begin{equation}
        S_{ij,\beta}(c)
        :=
        \sum_{k\in\mathcal N_{ij}}
        e^{\beta H_{ij,k}}|c^\top\widehat n_{ij,k}|.
        \label{eq:unnormalized-score}
\end{equation}

\subsection{Deterministic exact recovery}
\label{app:theory-det}

The deterministic theorem says that if every edge has enough clean-clean triangle
witnesses, if these witnesses are geometrically well-distributed, if the
point-support scores separate clean-clean triangles from the remaining triangles,
and if the local candidate pool contains the true direction, then one sufficiently
sharp \TriDE{} sweep selects the true direction on every edge.

\begin{definition}[Clean anchors and clean-clean witnesses]
\label{def:clean-witnesses}
Let \(R\subseteq E\) be the set of clean anchor edges at the beginning of the
sweep.  For a target edge \(ij\), define
\begin{equation*}
        \mathcal G_{ij}
        :=
        \{k\in\mathcal N_{ij}: ik\in R,\ jk\in R\},
        \qquad
        \mathcal B_{ij}:=\mathcal N_{ij}\setminus\mathcal G_{ij}.
\end{equation*}
The set \(\mathcal G_{ij}\) is the set of retained clean-clean witnesses for
\(ij\).
\end{definition}

\begin{definition}[Good-witness density]
\label{def:witness-density}
The anchor set \(R\) has good-witness density \(a\in(0,1]\) if
\begin{equation*}
        |\mathcal G_{ij}|\ge a|\mathcal N_{ij}|
        \qquad
        \text{for every }ij\in E.
\end{equation*}
In particular, every target edge has at least one retained clean-clean witness.
\end{definition}

\begin{definition}[Well-distributed clean-clean witnesses]
\label{def:wd-witnesses}
For \(c_{\rm wd}>0\), the set \(\mathcal G_{ij}\) is
\(c_{\rm wd}\)-well-distributed with respect to \(ij\) if
\begin{equation}
        \frac1{|\mathcal G_{ij}|}
        \sum_{k\in\mathcal G_{ij}}
        |h^\top n_{ij,k}^\star|
        \ge
        c_{\rm wd}\|P_{(g_{ij}^\star)^\perp}h\|_2
        \qquad
        \forall h\in\R^3.
        \label{eq:wd}
\end{equation}
\end{definition}

\begin{definition}[Support-sum separation]
\label{def:support-gap}
The sweep state has support-sum gap \(\Delta>0\) if, for every target edge
\(ij\in E\),
\begin{equation}
        \min_{k\in\mathcal G_{ij}}H_{ij,k}
        -
        \max_{k\in\mathcal B_{ij}}H_{ij,k}
        \ge
        \Delta.
        \label{eq:ss-gap}
\end{equation}
If \(\mathcal B_{ij}=\varnothing\), the maximum over \(\mathcal B_{ij}\) is
interpreted as \(-\infty\).
\end{definition}

\begin{definition}[Admissible candidate pool]
\label{def:admissible-candidates}
The candidate pool \(\mathcal C_{ij}\) is \(\eta\)-admissible if
\begin{equation*}
        \pm g_{ij}^\star\in\mathcal C_{ij}
\end{equation*}
and every non-true candidate is separated from the true unoriented direction:
\begin{equation*}
        c\in\mathcal C_{ij}\setminus\{\pm g_{ij}^\star\}
        \quad\Longrightarrow\quad
        \err(c,g_{ij}^\star)\ge\eta.
\end{equation*}
The separation condition is only needed to turn an error bound into exact equality
for a finite candidate pool.  Without it, Theorem~\ref{thm:deterministic-tride}
still gives the displayed deterministic error bound.
\end{definition}

\begin{theorem}[Deterministic one-sweep exact recovery]
\label{thm:deterministic-tride}
Assume the following conditions.
\begin{enumerate}
\item[(D1)] \textbf{Exact clean anchors.}  For every \(e\in R\),
\begin{equation*}
        g_e^{(0)}=\pm g_e^\star.
\end{equation*}
\item[(D2)] \textbf{Dense clean-clean witnesses.}  The anchor set \(R\) has
        good-witness density \(a\in(0,1]\).
\item[(D3)] \textbf{Well-distributed clean-clean witnesses.}  For every
        \(ij\in E\), the set \(\mathcal G_{ij}\) is
        \(c_{\rm wd}\)-well-distributed.
\item[(D4)] \textbf{Support-sum separation.}  The sweep state has support-sum
        gap \(\Delta>0\).
\item[(D5)] \textbf{Candidate admissibility.}  Every \(\mathcal C_{ij}\) is
        \(\eta\)-admissible.
\end{enumerate}
Then one synchronous \TriDE{} sweep satisfies
\begin{equation}
        \max_{ij\in E}\err(g_{ij}^+,g_{ij}^\star)
        \le
        \frac{1-a}{a\,c_{\rm wd}}e^{-\beta\Delta}.
        \label{eq:det-error-bound}
\end{equation}
Consequently, if
\begin{equation}
        \frac{1-a}{a\,c_{\rm wd}}e^{-\beta\Delta}<\eta,
        \label{eq:det-exact-condition}
\end{equation}
then
\begin{equation*}
        g_{ij}^+=\pm g_{ij}^\star
        \qquad
        \forall ij\in E.
\end{equation*}
Equivalently, it is sufficient to choose
\begin{equation*}
        \beta>
        \frac1\Delta
        \log\!\left(\frac{1-a}{a\,c_{\rm wd}\eta}\right),
\end{equation*}
with the convention that the logarithmic condition is vacuous if
\((1-a)/(a\,c_{\rm wd}\eta)\le1\).
\end{theorem}

\begin{proof}
Fix a target edge \(ij\).  Write
\(\mathcal G=\mathcal G_{ij}\) and \(\mathcal B=\mathcal B_{ij}\).  We compare
candidates using the unnormalized score~\eqref{eq:unnormalized-score}.

For \(k\in\mathcal G\), both neighboring edges are exact clean anchors by (D1).
Therefore
\begin{equation*}
        \widehat n_{ij,k}=\pm n_{ij,k}^\star.
\end{equation*}
Because the score uses absolute values, the sign is irrelevant, and the true
candidate has zero residual on every clean-clean witness:
\begin{equation*}
        |(g_{ij}^\star)^\top\widehat n_{ij,k}|=0,
        \qquad k\in\mathcal G.
\end{equation*}
Let
\begin{equation*}
        H_+ := \min_{k\in\mathcal G}H_{ij,k}.
\end{equation*}
By support-sum separation, every \(k\in\mathcal B\) satisfies
\begin{equation*}
        H_{ij,k}\le H_+-\Delta.
\end{equation*}
Every absolute residual is at most one.  Thus the true candidate satisfies
\begin{equation}
        S_{ij,\beta}(g_{ij}^\star)
        \le
        |\mathcal B|e^{\beta(H_+-\Delta)}.
        \label{eq:true-score-upper}
\end{equation}
If \(\mathcal B=\varnothing\), the right-hand side is zero.

Now take any candidate \(c\in\mathcal C_{ij}\).  The clean-clean witnesses alone
give
\begin{align}
        S_{ij,\beta}(c)
        &\ge
        \sum_{k\in\mathcal G}e^{\beta H_{ij,k}}
        |c^\top n_{ij,k}^\star| \notag                                            \\
        &\ge
        e^{\beta H_+}
        \sum_{k\in\mathcal G}|c^\top n_{ij,k}^\star| \notag                         \\
        &\ge
        e^{\beta H_+}|\mathcal G|c_{\rm wd}\,
        \err(c,g_{ij}^\star),
        \label{eq:false-score-lower}
\end{align}
where the last line uses the well-distributedness condition~\eqref{eq:wd} with
\(h=c\).

Since \(g_{ij}^\star\in\mathcal C_{ij}\), the selected candidate obeys
\begin{equation*}
        S_{ij,\beta}(g_{ij}^+)
        \le
        S_{ij,\beta}(g_{ij}^\star).
\end{equation*}
Combining this inequality with~\eqref{eq:true-score-upper} and
\eqref{eq:false-score-lower} applied to \(c=g_{ij}^+\), we get
\begin{equation*}
        e^{\beta H_+}|\mathcal G|c_{\rm wd}\,
        \err(g_{ij}^+,g_{ij}^\star)
        \le
        |\mathcal B|e^{\beta(H_+-\Delta)}.
\end{equation*}
Therefore
\begin{equation*}
        \err(g_{ij}^+,g_{ij}^\star)
        \le
        \frac{|\mathcal B|}{|\mathcal G|}\frac1{c_{\rm wd}}e^{-\beta\Delta}.
\end{equation*}
By good-witness density,
\begin{equation*}
        \frac{|\mathcal B|}{|\mathcal G|}
        =
        \frac{|\mathcal N_{ij}|-|\mathcal G_{ij}|}{|\mathcal G_{ij}|}
        \le
        \frac{1-a}{a}.
\end{equation*}
This proves~\eqref{eq:det-error-bound}.  If
\eqref{eq:det-exact-condition} holds, then every selected candidate has error
strictly below \(\eta\).  By candidate admissibility, no non-true candidate has
error below \(\eta\), so \(g_{ij}^+=\pm g_{ij}^\star\).  Since the edge \(ij\) was
arbitrary, the conclusion holds for all observed edges.
\end{proof}

\begin{corollary}[Edge-level support condition implies support-sum separation]
\label{cor:support-gap-from-edges}
Suppose there are numbers \(A_+>A_-\) and \(u\ge0\) such that
\begin{equation*}
        A_e^{(0)}\in[A_+-u,A_++u]
        \qquad
        \forall e\in R,
\end{equation*}
and
\begin{equation*}
        A_e^{(0)}\le A_-+u
        \qquad
        \forall e\notin R.
\end{equation*}
Then the support-sum separation condition~\eqref{eq:ss-gap} holds with
\begin{equation*}
        \Delta=A_+-A_- -4u.
\end{equation*}
In particular, if \(A_+-A_->4u\), Theorem~\ref{thm:deterministic-tride} applies
with the positive effective gap \(\Delta_{\rm eff}=A_+-A_- -4u\).
\end{corollary}

\begin{proof}
If \(k\in\mathcal G_{ij}\), then both neighboring edges are in \(R\), so
\begin{equation*}
        H_{ij,k}=A_{ik}^{(0)}+A_{jk}^{(0)}\ge 2(A_+-u).
\end{equation*}
If \(k\in\mathcal B_{ij}\), then at least one of \(ik\) and \(jk\) is not in
\(R\).  That nonclean edge has support at most \(A_-+u\), while the other edge
has support at most \(A_++u\) if it is clean and at most \(A_-+u\le A_++u\)
otherwise.  Hence
\begin{equation*}
        H_{ij,k}\le (A_++u)+(A_-+u),
        \qquad k\in\mathcal B_{ij}.
\end{equation*}
The clean-clean support sum exceeds the non-clean-clean support sum by at least
\begin{equation*}
        2(A_+-u)-(A_++u+A_-+u)=A_+-A_- -4u.
\end{equation*}
\end{proof}

\paragraph{Exact states are fixed points.}
If all retained witnesses for every edge are clean, well-distributed, and
nondegenerate, then \(a=1\) and the deterministic bound gives zero error.  Thus,
provided the true direction remains in each candidate pool, an exact direction
field is a fixed point of the update.

\subsection{Probabilistic ingredients}
\label{app:theory-prob-lemmas}

The deterministic theorem only needs realized support separation,
well-distributed clean-clean witnesses, and candidate admissibility.  The lemmas
below give simple sufficient conditions for these assumptions in the random
models.

\subsubsection{Point-support gap at fixed scale}

Let \(u\in\Sph^2\) be fixed and let \(x\sim\operatorname{Unif}(\Sph^2)\).  Since
\(|u^\top x|\sim\operatorname{Unif}[0,1]\) in three dimensions, the background
support of a random correspondence normal is
\begin{equation*}
        b_{\sigma_0}
        =
        \int_0^1
        \exp\!\left(-\frac{\arcsin^2 z}{2\sigma_0^2}\right)dz
        =
        \int_0^{\pi/2}e^{-\alpha^2/(2\sigma_0^2)}\cos\alpha\,d\alpha.
\end{equation*}
For small \(\sigma_0\) in radians,
\begin{equation*}
        b_{\sigma_0}
        =
        \sqrt{\frac\pi2}\,\sigma_0+O(\sigma_0^3),
\end{equation*}
and at \(\sigma_0=1^\circ\), \(b_{\sigma_0}\approx0.022\).

If a clean edge has an exact-inlier fraction \(\pi_{\rm cl}\), exact inliers
satisfy \((g_e^\star)^\top x=0\) and contribute one to \(A_e(g_e^\star)\).  Under
uniform background outliers, the population support of the true direction is
\begin{equation*}
        A_+
        =
        \pi_{\rm cl}+(1-\pi_{\rm cl})b_{\sigma_0}.
\end{equation*}
A generic wrong initialized direction has population support close to
\(A_-\approx b_{\sigma_0}\), giving the population gap
\begin{equation*}
        \Delta_{\sigma_0}:=A_+-A_-
        \approx
        \pi_{\rm cl}(1-b_{\sigma_0}).
\end{equation*}
Thus the theory does not require majority inliers on a clean anchor edge; it
requires a positive support gap between clean initialized directions and weak
initialized directions.

\begin{lemma}[Uniform support concentration]
\label{lem:support-concentration}
Let \(m=|E|\).  Let \(\mathcal F_0\) be the sigma-field that determines the graph,
the clean-anchor indicators, and the initialized directions
\(\{g_e^{(0)}\}_{e\in E}\).  Conditional on \(\mathcal F_0\), assume that the
measurements used to evaluate \(A_e(g_e^{(0)})\) are independent across samples
and that each summand in~\eqref{eq:theory-point-support} lies in \([0,1]\).  Suppose
\(N_e\ge N_{\min}\) for all edges and set
\begin{equation*}
        \mu_e:=\Exp[A_e(g_e^{(0)})\mid\mathcal F_0].
\end{equation*}
Then, conditional on \(\mathcal F_0\), with probability at least \(1-\zeta\),
\begin{equation*}
        |A_e^{(0)}-\mu_e|
        \le
        u_N:=\sqrt{\frac{\log(2m/\zeta)}{2N_{\min}}}
        \qquad
        \forall e\in E.
\end{equation*}
\end{lemma}

\begin{proof}
For a fixed edge, Hoeffding's inequality gives
\begin{equation*}
        \Prob\{|A_e^{(0)}-\mu_e|>u\mid\mathcal F_0\}
        \le
        2e^{-2N_eu^2}
        \le
        2e^{-2N_{\min}u^2}.
\end{equation*}
Taking a union bound over \(m\) edges and choosing
\(u=\sqrt{\log(2m/\zeta)/(2N_{\min})}\) proves the claim.
\end{proof}

\paragraph{Split and non-split measurements.}
The concentration lemma is a sufficient probabilistic justification for the
support gap when the measurements used to score an initialized direction are
independent of that initialized direction, for example by sample splitting.  The
deterministic theorem itself does not require sample splitting; it only requires
the realized support-sum separation condition~\eqref{eq:ss-gap}.

\subsubsection{Candidate recall}

\begin{lemma}[Random-pair candidate recall]
\label{lem:candidate-recall}
Suppose every edge has at least a \(\pi_{\min}\)-fraction of exact inlier
correspondence normals, and two independently selected inlier normals on the same
edge are nonparallel with probability one.  If the candidate pool for each edge
contains \(B\) independent random two-normal hypotheses, then there is a numerical
constant \(c>0\) such that
\begin{equation*}
        \Prob\{\text{some observed edge has no inlier--inlier sampled pair}\}
        \le
        m e^{-cB\pi_{\min}^2}.
\end{equation*}
Consequently, \(B\pi_{\min}^2\ge C\log(m/\zeta)\) is sufficient for every edge's
candidate pool to contain \(\pm g_e^\star\) with probability at least
\(1-\zeta\).
\end{lemma}

\begin{proof}
For one edge, a random pair is an inlier--inlier pair with probability at least
\(c\pi_{\min}^2\), where \(c\) absorbs the distinction between ordered,
unordered, with-replacement, and without-replacement sampling.  Two nonparallel
exact inlier normals lie in \((g_e^\star)^\perp\), and their cross product is
therefore parallel to \(g_e^\star\).  Hence an inlier--inlier pair contributes the
true unoriented direction to the candidate pool.  The probability that \(B\) pair
samples miss all inlier--inlier pairs is at most
\begin{equation*}
        (1-c\pi_{\min}^2)^B\le e^{-cB\pi_{\min}^2}.
\end{equation*}
A union bound over the \(m\) observed edges proves the result.
\end{proof}

\subsubsection{Empirical well-distributedness}

For distinct points \(x,y\in\R^3\), define
\begin{equation*}
        \mathcal W_{x,y}(z,h)
        :=
        \left\|P_{\spanop\{z-x,z-y\}^{\perp}}h\right\|_2.
\end{equation*}
In three dimensions, when \(x,y,z\) are noncollinear,
\(\mathcal W_{x,y}(z,h)=|h^\top n_{xy,z}|\), where \(n_{xy,z}\) is the unit
normal to the plane through \(x,y,z\).

\begin{definition}[Population well-distributedness]
\label{def:population-wd}
A distribution \(\mu\) on \(\R^3\) is \(c_0\)-well-distributed if, for all distinct
\(x,y\) in its support and all \(h\in\R^3\),
\begin{equation*}
        \Exp_{Z\sim\mu}\mathcal W_{x,y}(Z,h)
        \ge
        c_0\|P_{(x-y)^\perp}h\|_2.
\end{equation*}
\end{definition}

\begin{lemma}[Empirical well-distributedness]
\label{lem:empirical-wd}
Let \(Z_1,\ldots,Z_M\) be i.i.d. samples from a
\(c_0\)-well-distributed distribution \(\mu\), conditional on fixed distinct
\(x,y\).  There is a numerical constant \(C\) such that if
\begin{equation}
        M\ge Cc_0^{-2}\log(L/\zeta),
        \label{eq:emp-wd-sample-size}
\end{equation}
then, with probability at least \(1-\zeta/L\),
\begin{equation*}
        \frac1M\sum_{\ell=1}^M\mathcal W_{x,y}(Z_\ell,h)
        \ge
        \frac{c_0}{2}\|P_{(x-y)^\perp}h\|_2
        \qquad
        \forall h\in\R^3.
\end{equation*}
\end{lemma}

\begin{proof}
By homogeneity it is enough to consider unit vectors
\(h\in(x-y)^\perp\), which form a unit circle.  For fixed \(h\), the random
variables \(Y_\ell(h):=\mathcal W_{x,y}(Z_\ell,h)\) are bounded in \([0,1]\) and
have mean at least \(c_0\).  Hoeffding's inequality gives
\begin{equation*}
        \Prob\!\left\{
        \frac1M\sum_{\ell=1}^M Y_\ell(h)<\frac{3c_0}{4}
        \right\}
        \le
        e^{-cMc_0^2}
\end{equation*}
for a numerical constant \(c>0\).  Apply this bound on a fixed
\((c_0/16)\)-net of the unit circle.  The net has cardinality \(O(c_0^{-1})\).
The map \(h\mapsto \mathcal W_{x,y}(z,h)\) is 1-Lipschitz for every \(z\), so the
average extends from the net to the whole unit circle with loss at most
\(c_0/4\).  Taking \(C\) sufficiently large in~\eqref{eq:emp-wd-sample-size}
proves the claim.
\end{proof}

\subsection{Erd\H{o}s--R\'enyi phase transition}
\label{app:theory-er}

The following theorem verifies the deterministic assumptions in an
Erd\H{o}s--R\'enyi viewing graph.  The clean edges are not the only edges to be
recovered.  They serve as anchors that provide clean-clean triangle witnesses;
weak edges may be initialized incorrectly but must still contain the true
direction in their local candidate pools.

\begin{theorem}[Erd\H{o}s--R\'enyi exact-recovery phase transition]
\label{thm:er-phase}
Let \(t_1,\ldots,t_n\) be i.i.d. from a compact
\(c_0\)-well-distributed distribution on \(\R^3\).  Let
\(G_n\sim G(n,p)\) independently of the locations.  Conditional on being observed,
each edge is independently clean with probability \(1-q\); let \(R\) be the set
of clean observed edges.  Assume:
\begin{enumerate}
\item[(E1)] If \(e\in R\), then \(g_e^{(0)}=\pm g_e^\star\) and the conditional
        population support of \(g_e^{(0)}\) is \(A_+\).  If \(e\notin R\), the
        conditional population support of \(g_e^{(0)}\) is at most \(A_-\), where
        \(\Delta_{\sigma_0}:=A_+-A_->0\).
\item[(E2)] The support-evaluation measurements satisfy the independence and
        boundedness assumptions of Lemma~\ref{lem:support-concentration}, with at
        least \(N_{\min}\) measurements per observed edge.
\item[(E3)] Every observed edge has at least a \(\pi_{\min}\)-fraction of exact
        inlier correspondence normals for candidate generation, and each edge
        samples \(B\) random two-normal hypotheses.
\item[(E4)] On the candidate-recall event, every non-true candidate in every pool
        is separated from the truth by at least \(\eta>0\) in the error
        metric~\eqref{eq:theory-error}.
\item[(E5)] Clean-clean true triangle contexts are nondegenerate and are retained
        by the update.  With a zero degeneracy threshold this holds almost surely
        for continuous nondegenerate location distributions.
\end{enumerate}
Fix \(\zeta\in(0,1)\) and define
\begin{equation*}
        u_n:=\sqrt{\frac{\log(2n^2/\zeta)}{2N_{\min}}},
        \qquad
        \Delta_{\rm eff}:=\Delta_{\sigma_0}-4u_n.
\end{equation*}
Assume \(\Delta_{\rm eff}>0\) and
\begin{align}
        n p^2(1-q)^2
        &\ge
        Cc_0^{-2}\log(n^2/\zeta),
        \label{eq:er-coverage-cond}\\
        N_{\min}\Delta_{\sigma_0}^2
        &\ge
        C\log(n^2/\zeta),
        \label{eq:er-support-cond}\\
        B\pi_{\min}^2
        &\ge
        C\log(n^2/\zeta),
        \label{eq:er-candidate-cond}
\end{align}
for a sufficiently large numerical constant \(C\).  If
\begin{equation}
        \frac{8}{c_0(1-q)^2}e^{-\beta\Delta_{\rm eff}}<\eta,
        \label{eq:er-beta-cond}
\end{equation}
then one \TriDE{} sweep exactly recovers all observed directions with probability
at least \(1-C\zeta\):
\begin{equation*}
        g_{ij}^+=\pm g_{ij}^\star
        \qquad
        \forall ij\in E(G_n).
\end{equation*}
Consequently, when \(c_0\), \(\Delta_{\sigma_0}\), and \(\eta\) are bounded below
by positive constants and \(\beta\) is chosen to satisfy~\eqref{eq:er-beta-cond},
the graph-coverage threshold is
\begin{equation*}
        n p^2(1-q)^2 \gtrsim \log n.
\end{equation*}
Equivalently, for a given \(p\), a sufficient scaling is
\begin{equation*}
        q
        \le
        1-C\sqrt{\frac{\log n}{n p^2}}.
\end{equation*}
For constant \(p\), the initially clean fraction may be as small as
\(1-q\asymp\sqrt{\log n/n}\).
\end{theorem}

\begin{proof}
We verify the assumptions of Theorem~\ref{thm:deterministic-tride} uniformly over
all observed edges.

Fix an unordered pair \(ij\).  Conditional on \(ij\in E(G_n)\), a vertex \(k\) is
a clean-clean witness if \(ik\in E(G_n)\), \(jk\in E(G_n)\), and both supporting
edges are clean.  Hence
\begin{equation*}
        |\mathcal G_{ij}|
        \sim
        \operatorname{Binomial}(n-2,p^2(1-q)^2),
\end{equation*}
for the all-common-neighbor model; retained triangles can only remove bad
triangles, while assumption (E5) keeps the clean-clean ones.  Chernoff's inequality
and a union bound over at most \(n^2\) pairs imply that, under
\eqref{eq:er-coverage-cond}, with probability at least \(1-\zeta\),
\begin{equation}
        |\mathcal G_{ij}|
        \ge
        \frac12(n-2)p^2(1-q)^2
        \qquad
        \forall ij\in E(G_n).
        \label{eq:er-good-count}
\end{equation}
The total number of common neighbors is
\(\operatorname{Binomial}(n-2,p^2)\).  The same Chernoff and union-bound argument
also gives
\begin{equation}
        |\mathcal N_{ij}|
        \le
        2(n-2)p^2
        \qquad
        \forall ij\in E(G_n)
        \label{eq:er-total-count}
\end{equation}
with probability at least \(1-\zeta\); if the implementation discards some
nonclean degenerate triangles, \(|\mathcal N_{ij}|\) is only smaller.  Combining
\eqref{eq:er-good-count} and~\eqref{eq:er-total-count}, the deterministic
witness-density condition holds with
\begin{equation*}
        a=\frac{(1-q)^2}{4}.
\end{equation*}

Conditional on the clean-clean witness set size, the corresponding third-camera
locations are i.i.d. from the same location distribution, because the ER graph and
clean indicators are independent of the locations.  Lemma~\ref{lem:empirical-wd},
with \(L=n^2\), and the lower bound~\eqref{eq:er-good-count} imply that all
\(\mathcal G_{ij}\)'s are \((c_0/2)\)-well-distributed simultaneously with
probability at least \(1-\zeta\).  Thus
\begin{equation*}
        c_{\rm wd}=c_0/2.
\end{equation*}

By Lemma~\ref{lem:support-concentration}, using \(m\le n^2\), condition
\eqref{eq:er-support-cond} gives
\begin{equation*}
        |A_e^{(0)}-\mu_e|\le u_n
        \qquad
        \forall e\in E(G_n)
\end{equation*}
with probability at least \(1-\zeta\).  Assumption (E1) and
Corollary~\ref{cor:support-gap-from-edges} then give support-sum separation with
\begin{equation*}
        \Delta=\Delta_{\rm eff}=\Delta_{\sigma_0}-4u_n.
\end{equation*}

By Lemma~\ref{lem:candidate-recall}, condition~\eqref{eq:er-candidate-cond}
ensures that all candidate pools contain the corresponding true directions with
probability at least \(1-\zeta\).  Assumption (E4) gives the remaining
\(\eta\)-separation part of candidate admissibility.

On the intersection of these events, Theorem~\ref{thm:deterministic-tride}
applies with \(a=(1-q)^2/4\), \(c_{\rm wd}=c_0/2\), and
\(\Delta=\Delta_{\rm eff}\).  Therefore
\begin{equation*}
        \max_{ij\in E(G_n)}\err(g_{ij}^+,g_{ij}^\star)
        \le
        \frac{1-a}{a\,c_{\rm wd}}e^{-\beta\Delta_{\rm eff}}
        \le
        \frac{8}{c_0(1-q)^2}e^{-\beta\Delta_{\rm eff}}.
\end{equation*}
Condition~\eqref{eq:er-beta-cond} makes the right-hand side smaller than
\(\eta\), so candidate admissibility forces exact recovery on every observed
edge.  The union of the failure probabilities is at most \(C\zeta\).
\end{proof}

\begin{corollary}[Complete graph]
\label{cor:complete-graph}
The complete graph case is obtained from Theorem~\ref{thm:er-phase} by setting
\(p=1\).  The clean-clean witness coverage condition becomes
\begin{equation*}
        n(1-q)^2\gtrsim \log n,
\end{equation*}
so a sufficient scaling is
\begin{equation*}
        q\le 1-C\sqrt{\frac{\log n}{n}},
\end{equation*}
provided the support-gap, candidate-recall, well-distributedness, and
\(\beta\)-sharpness conditions also hold.
\end{corollary}

\subsection{Random geometric graph phase transition}
\label{app:theory-rgg}

We next state the corresponding result for a random geometric graph in three
dimensions.  This model captures the fact that nearby cameras are more likely to
be connected, and it replaces the ER common-neighbor scale \(np^2\) by the lens
scale \(nr^3\).

Let \(\Omega\subset\R^3\) be compact and let \(\mu\) be a probability distribution
supported on \(\Omega\).  Camera locations are i.i.d.:
\begin{equation*}
        t_1,\ldots,t_n\stackrel{\rm i.i.d.}{\sim}\mu.
\end{equation*}
The random geometric graph \(G_{\rm RGG}(n,r)\) connects \(ij\) when
\begin{equation*}
        \|t_i-t_j\|_2\le r.
\end{equation*}
For \(x,y\in\Omega\), define the common-neighborhood lens
\begin{equation*}
        L(x,y;r):=\Omega\cap B(x,r)\cap B(y,r).
\end{equation*}

\begin{assumption}[Lens regularity]
\label{ass:lens-regularity}
There are constants \(r_0>0\), \(0<v_-\le v_+<\infty\), and
\(c_{\rm lens}>0\) such that, for every \(0<r\le r_0\) and every
\(x,y\in\Omega\) with \(0<\|x-y\|_2\le r\):
\begin{enumerate}
\item[(L1)] \textbf{Lens mass.}
\begin{equation*}
        v_-r^3\le \mu(L(x,y;r))\le v_+r^3.
\end{equation*}
\item[(L2)] \textbf{Lens well-distributedness.}  If
        \(Z\sim\mu(\cdot\mid L(x,y;r))\), then for every \(h\in\R^3\),
\begin{equation*}
        \Exp\,\mathcal W_{x,y}(Z,h)
        \ge
        c_{\rm lens}\|P_{(x-y)^\perp}h\|_2.
\end{equation*}
\end{enumerate}
\end{assumption}

\begin{theorem}[Random-geometric exact-recovery phase transition]
\label{thm:rgg-phase}
Let \(t_1,\ldots,t_n\stackrel{\rm i.i.d.}{\sim}\mu\), where \(\mu\) satisfies
Assumption~\ref{ass:lens-regularity}.  Let \(G_{\rm RGG}(n,r)\) connect pairs
within radius \(r\le r_0\).  Conditional on being observed, each edge is
independently clean with probability \(1-q\).  Assume the same support model,
support-evaluation concentration model, candidate-recall model, candidate
separation \(\eta\), and clean-triangle retention condition as in
Theorem~\ref{thm:er-phase}, with population gap
\(\Delta_{\sigma_0}=A_+-A_->0\).

Fix \(\zeta\in(0,1)\) and define
\begin{equation*}
        u_n:=\sqrt{\frac{\log(2n^2/\zeta)}{2N_{\min}}},
        \qquad
        \Delta_{\rm eff}:=\Delta_{\sigma_0}-4u_n.
\end{equation*}
Assume \(\Delta_{\rm eff}>0\) and
\begin{align}
        n r^3(1-q)^2
        &\ge
        C\log(n^2/\zeta),
        \label{eq:rgg-coverage-cond}\\
        N_{\min}\Delta_{\sigma_0}^2
        &\ge
        C\log(n^2/\zeta),
        \label{eq:rgg-support-cond}\\
        B\pi_{\min}^2
        &\ge
        C\log(n^2/\zeta),
        \label{eq:rgg-candidate-cond}
\end{align}
where \(C\) may depend on \(v_-\), \(v_+\), and \(c_{\rm lens}\).  If
\begin{equation}
        \frac{8v_+}{v_-c_{\rm lens}(1-q)^2}
        e^{-\beta\Delta_{\rm eff}}
        <
        \eta,
        \label{eq:rgg-beta-cond}
\end{equation}
then one \TriDE{} sweep exactly recovers every observed direction with probability
at least \(1-C\zeta\):
\begin{equation*}
        g_{ij}^+=\pm g_{ij}^\star
        \qquad
        \forall ij\in E(G_{\rm RGG}(n,r)).
\end{equation*}
Consequently, with fixed geometric, support-gap, and candidate-separation
constants and with \(\beta\) satisfying~\eqref{eq:rgg-beta-cond}, the graph
coverage threshold is
\begin{equation*}
        n r^3(1-q)^2\gtrsim \log n.
\end{equation*}
Equivalently,
\begin{equation*}
        r(1-q)^{2/3}
        \gtrsim
        \left(\frac{\log n}{n}\right)^{1/3}.
\end{equation*}
\end{theorem}

\begin{proof}
Fix an observed edge \(ij\).  Conditional on \(t_i,t_j\), a third vertex \(k\) is
a triangle witness exactly when
\begin{equation*}
        t_k\in L(t_i,t_j;r).
\end{equation*}
It is a clean-clean witness if, in addition, both observed supporting edges are
clean.  Thus, conditional on \(t_i,t_j\),
\begin{equation*}
        |\mathcal G_{ij}|
        \sim
        \operatorname{Binomial}\bigl(n-2,(1-q)^2\mu(L(t_i,t_j;r))\bigr),
\end{equation*}
again using that clean-clean triangles are retained.  By the lower lens-mass bound,
\begin{equation*}
        \Exp[|\mathcal G_{ij}|\mid t_i,t_j]
        \ge
        (n-2)(1-q)^2v_-r^3.
\end{equation*}
Chernoff's inequality and a union bound over all target pairs imply that, under
\eqref{eq:rgg-coverage-cond}, with probability at least \(1-\zeta\),
\begin{equation}
        |\mathcal G_{ij}|
        \ge
        \frac12 n(1-q)^2v_-r^3
        \qquad
        \forall ij\in E(G_{\rm RGG}).
        \label{eq:rgg-good-count}
\end{equation}
The total number of common neighbors has conditional mean at most
\((n-2)v_+r^3\), so the same argument gives
\begin{equation*}
        |\mathcal N_{ij}|
        \le
        2nv_+r^3
        \qquad
        \forall ij\in E(G_{\rm RGG})
\end{equation*}
with probability at least \(1-\zeta\).  Therefore
\begin{equation*}
        \frac{|\mathcal G_{ij}|}{|\mathcal N_{ij}|}
        \ge
        \frac{v_-}{4v_+}(1-q)^2,
        \qquad
        \forall ij\in E(G_{\rm RGG}),
\end{equation*}
so the deterministic density parameter is
\begin{equation*}
        a=\frac{v_-}{4v_+}(1-q)^2.
\end{equation*}

Conditional on \(t_i,t_j\) and the clean-clean witness count, the clean-clean
third-camera locations are i.i.d. from
\(\mu(\cdot\mid L(t_i,t_j;r))\).  The lens well-distributedness assumption and
Lemma~\ref{lem:empirical-wd}, combined with~\eqref{eq:rgg-good-count}, imply that
all clean-clean witness sets are \((c_{\rm lens}/2)\)-well-distributed
simultaneously with probability at least \(1-\zeta\).  Hence
\begin{equation*}
        c_{\rm wd}=c_{\rm lens}/2.
\end{equation*}

The support-concentration and candidate-recall events follow from
\eqref{eq:rgg-support-cond} and~\eqref{eq:rgg-candidate-cond} exactly as in the ER
proof.  Therefore Corollary~\ref{cor:support-gap-from-edges} gives
\(\Delta=\Delta_{\rm eff}\), and all candidate pools are \(\eta\)-admissible.

Applying Theorem~\ref{thm:deterministic-tride} with
\(a=v_-(1-q)^2/(4v_+)\), \(c_{\rm wd}=c_{\rm lens}/2\), and
\(\Delta=\Delta_{\rm eff}\), we obtain
\begin{equation*}
        \max_{ij\in E(G_{\rm RGG})}\err(g_{ij}^+,g_{ij}^\star)
        \le
        \frac{8v_+}{v_-c_{\rm lens}(1-q)^2}
        e^{-\beta\Delta_{\rm eff}}.
\end{equation*}
Condition~\eqref{eq:rgg-beta-cond} makes this bound smaller than \(\eta\), so the
selected candidate on every observed edge must be the true unoriented direction.
The total failure probability is at most \(C\zeta\).
\end{proof}

\subsection{Summary of sufficient graph scalings}
\label{app:theory-scaling-summary}

The deterministic theorem separates graph coverage from local candidate recall
and point-support separation.  Once the support-gap, candidate-admissibility,
well-distributedness, and sufficiently large-\(\beta\) conditions hold, the
clean-clean witness coverage required for one-sweep exact direction recovery is
\begin{align*}
        \text{complete graph:} &\qquad n(1-q)^2 \gtrsim \log n, \\
        \text{ER }G(n,p):      &\qquad n p^2(1-q)^2 \gtrsim \log n, \\
        \text{3D RGG:}        &\qquad n r^3(1-q)^2 \gtrsim \log n.
\end{align*}
The local candidate-recall condition from random two-normal hypotheses is
\begin{equation*}
        B\pi_{\min}^2\gtrsim \log |E|,
\end{equation*}
and the empirical point-support condition is
\begin{equation*}
        N_{\min}\Delta_{\sigma_0}^2\gtrsim \log |E|.
\end{equation*}
The role of \(\beta\) is to make clean-clean witnesses dominate the remaining
triangles.  For example, in the ER model a sufficient sharpness condition is
\begin{equation*}
        \beta>
        \frac1{\Delta_{\rm eff}}
        \log\!\left(\frac{8}{c_0(1-q)^2\eta}\right),
\end{equation*}
and the analogous RGG condition is given in~\eqref{eq:rgg-beta-cond}.

\section{Broader Impacts}
\label{app:broader-impacts}

This work studies pairwise translation-direction estimation for global structure-from-motion. Potential positive impacts include more robust 3D reconstruction, camera localization, mapping, and digitization of real-world scenes. Our experiments use public benchmark data and do not involve human-subject data collection, biometric recognition, or deployment of a real-world sensing system. Responsible use should respect dataset licenses, privacy expectations, and consent when reconstructing real environments.

\end{document}